\documentclass[review,authoryear]{elsarticle}\usepackage[]{graphicx}\usepackage[]{color}
\makeatletter
\def\maxwidth{ %
  \ifdim\Gin@nat@width>\linewidth
    \linewidth
  \else
    \Gin@nat@width
  \fi
}
\makeatother

\definecolor{fgcolor}{rgb}{0.345, 0.345, 0.345}

\usepackage{framed}
\makeatletter
 {\par\unskip\endMakeFramed%
 \at@end@of@kframe}
\makeatother

\definecolor{shadecolor}{rgb}{.97, .97, .97}
\definecolor{messagecolor}{rgb}{0, 0, 0}
\definecolor{warningcolor}{rgb}{1, 0, 1}
\definecolor{errorcolor}{rgb}{1, 0, 0}
\newenvironment{knitrout}{}{} 

\usepackage{alltt}

\usepackage{graphicx, amsmath, amssymb, bm, mathtools, amsthm, natbib}
\usepackage{algorithm2e, url}

\bibpunct{(}{)}{;}{a}{,}{,}

\newcommand{\xbar}{\bar{\bm x}}
\newcommand{\tr}{\text{tr}}
\DeclareMathOperator*{\argmin}{arg\,min}

\newtheorem{cor}{Corollary}
\newtheorem{lemma}{Lemma}
\newtheorem{proposition}{Proposition}
\newtheorem{thm}{Theorem}

\IfFileExists{upquote.sty}{\usepackage{upquote}}{}
\begin{document}

\begin{frontmatter}

\title{High-Dimensional Regularized Discriminant Analysis}

\author[jar]{John A. Ramey\cortext[cor1]{Corresponding author}}
\author[cks]{Caleb K. Stein}
\author[pdy]{Phil D. Young}
\author[dmy]{Dean M. Young}
\address[jar]{Novi Labs}
\address[cks]{Myeloma Institute, University of Arkansas for Medical Sciences}
\address[pdy]{Department of Management and Information Systems, Baylor University}
\address[dmy]{Department of Statistical Science, Baylor University}

\begin{abstract}

Regularized discriminant analysis (\emph{RDA}), proposed by
\cite{Friedman:1989tm}, is a widely popular classifier that lacks
interpretability and is impractical for high-dimensional data sets.  Here, we
present an interpretable and computationally efficient classifier called
high-dimensional \emph{RDA} (\emph{HDRDA}), designed for the small-sample,
high-dimensional setting. For \emph{HDRDA}, we show that each training
observation, regardless of class, contributes to the class covariance matrix,
resulting in an interpretable estimator that \emph{borrows} from the pooled
sample covariance matrix. Moreover, we show that \emph{HDRDA} is equivalent to
a classifier in a reduced-feature space with dimension approximately equal to
the training sample size. As a result, the matrix operations employed by
\emph{HDRDA} are computationally linear in the number of features, making the
classifier well-suited for high-dimensional classification in practice. We
demonstrate that \emph{HDRDA} is often superior to several sparse and
regularized classifiers in terms of classification accuracy with three
artificial and six real high-dimensional data sets. Also, timing comparisons
between our \emph{HDRDA} implementation in the {\tt sparsediscrim} R package
and the standard \emph{RDA} formulation in the {\tt klaR} R package demonstrate
that as the number of features increases, the computational runtime of
\emph{HDRDA} is drastically smaller than that of \emph{RDA}.

\end{abstract}


\end{frontmatter}

\section{Introduction}

In this paper, we consider the classification of small-sample, high-dimensional
data, where the number of features $p$ exceeds the training sample size $N$. In
this setting, well-established classifiers, such as linear discriminant analysis
(\emph{LDA}) and quadratic discriminant analysis (\emph{QDA}), become
incalculable because the class and pooled covariance matrix estimators are
singular \citep*{Murphy:2012uq, Bouveyron:2007gx, Mkhadri:1997gy}. To improve
the accuracy of the estimation of the class covariance matrices estimated in the
\emph{QDA} classifier and to ensure that the covariance matrix estimators are
nonsingular, \cite{Friedman:1989tm} proposed the regularized discriminant
analysis (\emph{RDA}) classifier by incorporating a weighted average of the
pooled sample covariance matrix and the class sample covariance matrix. To
further improve the accuracy of the estimation of the class covariance matrix
and to stabilize its inverse, \cite{Friedman:1989tm} also included a
regularization component by shrinking the covariance matrix estimator towards
the identity matrix, which yields a nonsingular estimator following the
well-known ridge-regression approach of \cite{Hoerl:1970cd}. Despite its
popularity, the ``borrowing'' operation employed in the \emph{RDA} classifier
lacks interpretability \citep{Bensmail:1996fw}. Furthermore, the \emph{RDA}
classifier is impractical for high-dimensional data sets because it computes the
inverse and determinant of the covariance matrices for each class. Both matrix
calculations are computationally expensive because the number of operations
grows at a polynomial rate in the number of features. Moreover, the model
selection of the \emph{RDA} classifier's two tuning parameters is
computationally burdensome because the matrix inverse and determinant of each
class covariance matrix are computed across multiple cross-validation folds for
each candidate tuning-parameter pair.

Here, we present the high-dimensional \emph{RDA} (\emph{HDRDA}) classifier,
which is intended for the case when $p > N$ . We reparameterize the \emph{RDA}
classifier similar to that of \cite*{Hastie:2008dt} and \cite{Halbe:2007ea} and
employ a biased covariance-matrix estimator that partially pools the individual
sample covariance matrices from the \emph{QDA} classifier with the pooled sample
covariance matrix from the \emph{LDA} classifier. We then shrink the resulting
covariance-matrix estimator towards a scaled identity matrix to ensure positive
definiteness. We show that the pooling parameter in the \emph{HDRDA} classifier
determines the contribution of each training observation to the estimation of
each class covariance matrix, enabling interpretability that has been previously
lacking with the \emph{RDA} classifier \citep{Bensmail:1996fw}. Our
parameterization differs from that of \cite*{Hastie:2008dt} and that of
\cite{Halbe:2007ea} in that our formulation allows the flexibility of various
covariance-matrix estimators proposed in the literature, including a variety of
ridge-like estimators, such as the one proposed by \cite{Srivastava:2007ww}.

Next, we establish that the matrix operations corresponding to the null space of
the pooled sample covariance matrix are redundant and can be discarded from the
\emph{HDRDA} decision rule without loss of classificatory information when we
apply reasoning similar to that of \cite{Ye:2006tq}. As a result, we achieve a
substantial reduction in dimension such that the matrix operations used in the
\emph{HDRDA} classifier are computationally linear in the number of features.
Furthermore, we demonstrate that the \emph{HDRDA} decision rule is invariant to
adjustments to the approximately $p - N$ zero eigenvalues, so that the decision
rule in the original feature space is equivalent to a decision rule in a lower
dimension, such that matrix inverses and determinants of relatively small
matrices can be rapidly computed. Finally, we show that several shrinkage
methods that are special cases of the \emph{HDRDA} classifier have no effect on
the approximately $p - N$ zero eigenvalues of the covariance-matrix estimators
when $p > N$. Such techniques include work from \cite{Srivastava:2007ww},
\cite{Rao:1971ul}, and other methods studied by \cite{Ramey:2013ji} and
\citet*{Xu:2009fl}.

We also provide an efficient algorithm along with pseudocode to estimate the
\emph{HDRDA} classifier's tuning parameters in a grid search via
cross-validation. Timing comparisons between our \emph{HDRDA} implementation in
the {\tt sparsediscrim} R package available on CRAN and the standard \emph{RDA}
formulation in the {\tt klaR} R package demonstrate that as the number of
features increases, the computational runtime of the \emph{HDRDA} classifier is
drastically smaller than that of \emph{RDA}. In fact, when $p = 5000$, we show
that the \emph{HDRDA} classifier is 502.786
times faster on average than the \emph{RDA} classifier. In this scenario, the
\emph{HDRDA} classifier's model selection requires 2.979
seconds on average, while that of the \emph{RDA} classifier requires
24.933 minutes on average.

Finally, we study the classification performance of the \emph{HDRDA} classifier
on six real high-dimensional data sets along with a simulation design that
generalizes the experiments initially conducted by \cite{Guo:2007te}. We
demonstrate that the \emph{HDRDA} classifier often attains superior
classification accuracy to several recent classifiers designed for small-sample,
high-dimensional data from \cite*{Tong:2012hw}, \cite{Witten:2011kc},
\cite*{Pang:2009ik}, and \cite{Guo:2007te}. We also include as a benchmark the
random forest from \cite{Breiman:2001fb} because \cite*{FernandezDelgado:2014ul}
have concluded that the random forest is often superior to other classifiers in
benchmark studies. We show that our proposed classifier is competitive and often
outperforms the random forest in terms of classification accuracy in the
small-sample, high-dimensional setting.

The remainder of this paper is organized as follows. In Section
\ref{sec:preliminaries} we introduce the classification problem and necessary
notation to describe our contributions. In Section \ref{sec:rda} we present the
\emph{HDRDA} classifier along with its interpretation. In Section
\ref{sec:hdrda-properties}, we provide properties of the \emph{HDRDA} classifier
and a computationally efficient model-selection procedure. In Section
\ref{sec:timing-comparisons}, we compare the model-selection timings of the
\emph{HDRDA} and \emph{RDA} classifiers. In Section \ref{sec:sims} we describe
our simulation studies of artificial and real data sets and examine the
experimental results. We conclude with a brief discussion in Section
\ref{sec:discussion}.

\section{Preliminaries}
\label{sec:preliminaries}

\subsection{Notation}

To facilitate our discussion of covariance-matrix regularization and
high-dimensional classification, we require the following notation. Let
$\mathbb{R}_{a \times b}$ denote the matrix space of all $a \times b$ matrices
over the real field $\mathbb{R}$. Denote by $\bm I_m$ the $m \times m$ identity
matrix, and let $\bm 0_{m \times p}$ be the $m \times p$ matrix of zeros, such
that $\bm 0_m$ is understood to denote $\bm 0_{m \times m}$. Define $\bm 1_m \in
\mathbb{R}_{m \times 1}$ as a vector of ones. Let $\bm A^{T}$, $\bm A^+$, and
$\mathcal{N}(\bm A)$ denote the transpose, the Moore-Penrose pseudoinverse, and
the null space of $\bm A \in \mathbb{R}_{m \times p}$, respectively. Denote by
$\mathbb{R}_{p \times p}^{>}$ the cone of real $p \times p$ positive-definite
matrices. Similarly, let $\mathbb{R}_{p \times p}^{\ge}$ denote the cone of real
$p \times p$ positive-semidefinite matrices. Let $V^{\perp}$ denote the
orthogonal complement of a vector space $V \subset \mathbb{R}_{p \times 1}$. For
$c \in \mathbb{R}$, let $c^+ = 1/c$ if $c \ne 0$ and $0$ otherwise.

\subsection{Discriminant Analysis}

In discriminant analysis we wish to assign an unlabeled vector $\bm x \in
\mathbb{R}_{p \times 1}$ to one of $K$ unique, known classes by constructing a
classifier from $N$ training observations. Let $\bm x_i = (x_{i1}, \ldots,
x_{ip}) \in \mathbb{R}_{p \times 1}$ be the $i$th observation $(i = 1, \ldots,
N)$ with true, unique membership $y_i \in \{\omega_1, \ldots, \omega_K\}$.
Denote by $n_k$ the number of training observations realized from class $k$,
such that $\sum_{k=1}^K n_k = N$. We assume that $(\bm x_i, y_i)$ is a
realization from a mixture distribution $p(\bm x) = \sum_{k=1}^K p(\bm x |
\omega_k) p(\omega_k)$, where $p(\bm x | \omega_k)$ is the probability density
function (PDF) of the $k$th class and $p(\omega_k)$ is the prior probability of
class membership of the $k$th class. We further assume $p(\omega_k) =
p(\omega_l)$, $1 \le k, l \le K$, $k \ne l$.

The \emph{QDA} classifier is the optimal Bayesian decision rule with respect to
a $0-1$ loss function when $p(\bm x | \omega_k)$ is the PDF of the multivariate
normal distribution with known mean vectors $\bm \mu_k \in \mathbb{R}_{p \times
  1}$ and known covariance matrices $\bm \Sigma_k \in \mathbb{R}_{p \times
  p}^{>}$, $k = 1, 2, \ldots, K$. Because $\bm \mu_k$ and $\bm \Sigma_k$ are
typically unknown, we assign an unlabeled observation $\bm x$ to class
$\omega_k$ with the sample \emph{QDA} classifier
\begin{align}
	D_{QDA}(\bm x) = \argmin_{k} (\bm x - \xbar_k)^{T}\widehat{\bm
    \Sigma}_k^{-1}(\bm x - \xbar_k) + \log |\widehat{\bm \Sigma}_k|, \label{eq:qda}
\end{align}
where $\xbar_k$ and $\widehat{\bm \Sigma}_k$ are the maximum-likelihood
estimators (MLEs) of $\bm \mu_k$ and $\bm \Sigma_k$, respectively. If we assume
further that $\bm \Sigma_k = \bm \Sigma$, $k = 1, \ldots, K$, then the pooled
sample covariance matrix $\widehat{\bm \Sigma}$ is substituted for $\widehat{
  \bm \Sigma}_k$ in \eqref{eq:qda}, where
\begin{align}
	\widehat{\bm \Sigma} = N^{-1} \sum_{k=1}^K n_k \widehat{\bm \Sigma}_k \label{eq:pooled-cov}
\end{align}
is the MLE for $\bm \Sigma$. Here, \eqref{eq:qda} reduces to the sample
\emph{LDA} classifier. We omit the log-determinant because it is constant across
the $K$ classes.

The smallest eigenvalues of $\widehat{\bm \Sigma}_k$ and the directions
associated with their eigenvectors can highly influence the classifier in
\eqref{eq:qda}. In fact, the eigenvalues of $\widehat{ \bm \Sigma}_k$ are
well-known to be biased if $p \ge n_k$ such that the smallest eigenvalues are
underestimated \citep{Seber:2004uh}. Moreover, if $p > n_k$, then
rank$(\widehat{\bm \Sigma}_k) \le n_k$, which implies that at least $p - n_k$
eigenvalues of $\widehat{\bm \Sigma}_k$ are zero. Furthermore, although more
feature information is available to discriminate among the $K$ classes, if $p >
n_k$, \eqref{eq:qda} is incalculable because $\widehat{\bm \Sigma}_k^{-1}$ does
not exist.

Several regularization methods, such as the methods considered by
\cite{Xu:2009fl}, \cite{Guo:2007te}, and \cite{Mkhadri:1995jp}, have been
proposed in the literature to adjust the eigenvalues of $\widehat{\bm \Sigma}_k$
so that \eqref{eq:qda} is calculable and provides reduced variability for
$\widehat{\bm \Sigma}_k^{-1}$. A common form of the covariance-matrix
regularization applies a shrinkage factor $\gamma > 0$, so that
\begin{align}
	\widehat{\bm \Sigma}_k(\gamma) = \widehat{\bm \Sigma}_k + \gamma \bm I_p, \label{eq:ridge-estimator}
\end{align}
similar to a method employed in ridge regression \citep{Hoerl:1970cd}. Equation
\eqref{eq:ridge-estimator} effectively \emph{shrinks} the sample covariance
matrix $\widehat{\bm \Sigma}_k$ toward $\bm I_p$, thereby increasing the
eigenvalues of $\widehat{\bm \Sigma}_k$ by $\gamma$. Specifically, the zero
eigenvalues are replaced with $\gamma$, so that \eqref{eq:ridge-estimator} is
positive definite. For additional covariance-matrix regularization methods, see
\cite{Ramey:2013ji}, \cite{Xu:2009fl}, and \cite{Ye:2009gd}.

\section{High-Dimensional Regularized Discriminant Analysis}
\label{sec:rda}

Here, we define the \emph{HDRDA} classifier by first formulating the
covariance-matrix estimator $\widehat{\bm \Sigma}_k(\lambda)$ and demonstrating
its clear interpretation as a linear combination of the crossproducts of the
training observations centered by their respective class sample means. We define
the convex combination
\begin{align}
  \widehat{\bm \Sigma}_k(\lambda) \coloneqq (1 - \lambda) \widehat{\bm \Sigma}_k + \lambda
  \widehat{\bm \Sigma}, \quad k = 1, \ldots, K,\label{eq:sig-lambda-alternative}
\end{align}
where $\lambda \in [0, 1]$ is the \emph{pooling} parameter. By rewriting
\eqref{eq:sig-lambda-alternative} in terms of the observations $\bm x_i$, $i =
1, \ldots, N$, each centered by its class sample mean, we attain a clear
interpretation of $\widehat{\bm \Sigma}_k(\lambda)$. That is,
\begin{align}
	\widehat{\bm \Sigma}_k(\lambda)
  &= \left( 1 - \lambda + \frac{\lambda n_k}{N} \right) \widehat{\bm \Sigma}_k +  \frac{\lambda}{N} \sum_{\substack{k' = 1\\k' \ne k}}^K n_{k'} \widehat{\bm \Sigma}_{k'} \nonumber \\
	&= \left( \frac{1 - \lambda}{n_k} + \frac{\lambda}{N} \right)\sum_{i=1}^N
  I(y_i = k) \bm x_i \bm x_i^{T} +  \frac{\lambda}{N} \sum_{i=1}^N I(y_i \ne k)
  \bm x_i \bm x_i^{T} \nonumber \\
	&= \sum_{i=1}^N c_{ik}(\lambda) \bm x_i \bm x_i^{T},\label{eq:sig-lambda-alternative2}
\end{align}
where $c_{ik}(\lambda) = \lambda N^{-1} + (1 - \lambda)n_k^{-1}I(y_i = k)$.
From \eqref{eq:sig-lambda-alternative2}, we see that $\lambda$ weights the
contribution of each of the $N$ observations in estimating $\bm \Sigma_k$ from
all $K$ classes rather than using only the $n_k$ observations from a single
class. As a result, we can interpret \eqref{eq:sig-lambda-alternative2} as a
covariance-matrix estimator that borrows from $\widehat{\bm \Sigma}$ in
\eqref{eq:pooled-cov} to estimate $ \bm \Sigma_k$.

In Figure \ref{fig:hdrda-contours} we plot the contours of five multivariate
normal populations for $\lambda = 0$ with unequal covariance matrices. As
$\lambda$ approaches 1, the contours become more similar, resulting in identical
contours for $\lambda = 1$. Below, we show that the pooling operation is
advantageous in increasing the rank of each $\widehat{\bm \Sigma}_k(\lambda)$
from rank$(\widehat{\bm \Sigma}_k)$ to rank$(\widehat{\bm \Sigma})$ for $0 <
\lambda \le 1$. Notice that if $\lambda = 0$, then the observations from the
remaining $K - 1$ classes do not contribute to the estimation of $\bm \Sigma_k$,
corresponding to $\widehat{\bm \Sigma}_k$. Furthermore, if $\lambda = 1$, the
weights $c_{ik}(\lambda)$ in \eqref{eq:sig-lambda-alternative2} reduce to $1/N$,
corresponding to $\widehat{\bm \Sigma}$. For brevity, when $\lambda = 1$, we
define $\bm X = [\sqrt{c_{1k}(1)} \bm x_1^{T}, \ldots, \sqrt{c_{Nk}(1)} \bm
  x_N^{T}]^{T}$ such that $\widehat{\bm \Sigma} = N^{-1} \bm X^{T} \bm
X$. Similarly, for $\lambda = 0$, we define $\bm X_k = [\sqrt{c_{1k}(0)} \bm
  x_1^{T}, \ldots, \sqrt{c_{Nk}(0)} \bm x_N^{T}]^{T}$ such that $\widehat{\bm
  \Sigma}_k = n_k^{-1} \bm X_k^{T} \bm X_k$.

\[ \left[\text{Insert Figure \ref{fig:hdrda-contours} approximately here }\right] \]

As we have discussed above, several eigenvalue adjustment methods have been
proposed that increase eigenvalues (approximately) equal to 0. To further
improve the estimation of $\bm \Sigma_k$ and to stabilize the estimator's
inverse, we define the eigenvalue adjustment of
\eqref{eq:sig-lambda-alternative} as
\begin{align}
	\tilde{\bm \Sigma}_k \coloneqq \alpha_k \widehat{\bm \Sigma}_k(\lambda) + \gamma \bm I_p,\label{eq:hdrda-cov}
\end{align}
where $\alpha_k \ge 0$ and $\gamma \ge 0$ is an eigenvalue-shrinkage
constant. Thus, the \emph{pooling} parameter $\lambda$ controls the amount of
estimation information borrowed from $\widehat{\bm \Sigma}$ to estimate $\bm
\Sigma_k$, and the \emph{shrinkage} parameter $\gamma$ determines the degree of
eigenvalue shrinkage. The choice of $\alpha_k$ allows for a flexible formulation
of covariance-matrix estimators. For instance, if $\alpha_k = 1$, $k = 1,
\ldots, K$, then \eqref{eq:hdrda-cov} resembles \eqref{eq:ridge-estimator}.
Similarly, if $\alpha_k = 1 - \gamma$, then \eqref{eq:hdrda-cov} has a form
comparable to the \emph{RDA} classifier from \cite{Friedman:1989tm}.
Substituting \eqref{eq:hdrda-cov} into \eqref{eq:qda}, we define the
\emph{HDRDA} classifier as
\begin{align}
	D_{HDRDA}(\bm x) = \argmin_{k}  (\bm x - \xbar_k)^{T}\tilde{\bm \Sigma}_k^{+}(\bm x - \xbar_k)  + \log |\tilde{\bm \Sigma}_k|. \label{eq:hdrda}
\end{align}
For $\gamma > 0$, $\tilde{\bm \Sigma}_k$ is nonsingular such that $\tilde{\bm
  \Sigma}_k^{-1}$ can be substituted for $\tilde{\bm \Sigma}_k^{+}$ in
\eqref{eq:hdrda}. If $\gamma = 0$, we explicitly set $|\tilde{\bm \Sigma}_k|$
equal to the product of the positive eigenvalues of $\tilde{\bm
  \Sigma}_k$. Following \cite{Friedman:1989tm}, we select $\lambda$ and $\gamma$
from a grid of candidate models via cross-validation \citep{Hastie:2008dt}. We provide an
implementation of \eqref{eq:hdrda} in the {\tt hdrda} function contained in the
{\tt sparsediscrim} R package, which is available on CRAN.

The choice of $\alpha_k$ in \eqref{eq:hdrda-cov} is one of convenience and
allows the flexibility of various covariance-matrix estimators proposed in the
literature. In practice, we generally are not interested in estimating
$\alpha_k$ because the estimation of $K$ additional tuning parameters via
cross-validation is counterproductive to our goal of computational efficiency.
For appropriate values of $\alpha_k$, the \emph{HDRDA} covariance-matrix
estimator includes or resembles a large family of estimators. Notice that if
$\alpha_k = 1$ and $\lambda = 1$, \eqref{eq:hdrda-cov} is equivalent to the
standard ridge-like covariance-matrix estimator in \eqref{eq:ridge-estimator}.
Other estimators proposed in the literature can be obtained when one selects
$\gamma$ accordingly. For instance, with $\gamma = \tr\{ \widehat{\bm \Sigma} \}
/ \min(N, p)$, we obtain the estimator from \cite{Srivastava:2007ww}.

When $\alpha_k = 1 - \gamma$, \eqref{eq:hdrda-cov} resembles the biased covariance-matrix estimator
\begin{align}
    \widehat{\bm \Sigma}_k(\lambda, \gamma) = (1 - \gamma) \widehat{\bm \Sigma}_k^{(RDA)}(\lambda) + \gamma \frac{\tr\left\{\widehat{\bm \Sigma}_k^{(RDA)}(\lambda)\right\}}{p} \bm I_p \label{eq:sig-rda}
\end{align}
employed in the \emph{RDA} classifier, where $\widehat{\bm
  \Sigma}_k^{(RDA)}(\lambda)$ is a pooled estimator of $\bm \Sigma_k$ and
$\gamma \in [0, 1]$ is a regularization parameter that controls the shrinkage of
\eqref{eq:sig-rda} towards $\bm I_p$ weighted by the average of the eigenvalues
of $\widehat{\bm \Sigma}_k^{(RDA)}(\lambda)$. Despite the similarity of
\eqref{eq:sig-rda} to the \emph{HDRDA} covariance-matrix estimator in
\eqref{eq:hdrda-cov}, the \emph{RDA} classifier is impractical for
high-dimensional data because the inverse and determinant of \eqref{eq:sig-rda}
must be calculated when substituted into \eqref{eq:qda}. Furthermore,
\eqref{eq:sig-rda} has no clear interpretation.

\section{Properties of the HDRDA Classifier}
\label{sec:hdrda-properties}

Next, we establish properties of the covariance-matrix estimator and the
decision rule employed in the \emph{HDRDA} classifier. By doing so, we
demonstrate that \eqref{eq:hdrda} lends itself to a more efficient
calculation. We decompose \eqref{eq:hdrda} into a sum of two components, where
the first summand consists of matrix operations applied to low-dimensional
matrices and the second summand corresponds to the null space of $\widehat{\bm
  \Sigma}$ in \eqref{eq:pooled-cov}. We show that the matrix operations
performed on the null space of $\widehat{\bm \Sigma}$ yield constant quadratic
forms across all classes and can be omitted. For $p \gg N$, the constant
component involves determinants and inverses of high-dimensional matrices, and
by ignoring these calculations, we achieve a substantial reduction in
computational costs. Furthermore, a byproduct is that adjustments to the
associated eigenvalues have no effect on \eqref{eq:hdrda}. Lastly, we utilize
the singular value decomposition to efficiently calculate the eigenvalue
decomposition of $\widehat{\bm \Sigma}$, further reducing the computational
costs of the \emph{HDRDA} classifier.

First, we require the following relationship regarding the null spaces of
$\widehat{\bm \Sigma}_k(\lambda)$, $\widehat{\bm \Sigma}$, and $\widehat{ \bm
  \Sigma}_k$.

\begin{lemma}\label{lemma:null-spaces}
Let $\widehat{\bm \Sigma}_k$ and $\widehat{\bm \Sigma}$ be the MLEs of $\bm
\Sigma_k$ and $\bm \Sigma$, respectively. Let $\widehat{\bm \Sigma}_k(\lambda)$
be defined as in \eqref{eq:sig-lambda-alternative}. Then,
$\mathcal{N}\{\widehat{\bm \Sigma}_k(\lambda)\} \subset \mathcal{N}(\widehat{
  \bm \Sigma}) \subset \mathcal{N}(\widehat{\bm \Sigma}_k)$, $k = 1, \ldots, K$.
\end{lemma}
\begin{proof}
Let $\bm z \in \mathcal{N}\{\widehat{\bm \Sigma}_k(\lambda)\}$ for some $k = 1,
\ldots, K$. Hence, $0 = \bm z^{T} \widehat{\bm \Sigma}_k(\lambda) \bm z = (1 -
\lambda) \bm z^{T} \widehat{\bm \Sigma}_k \bm z + \lambda \bm z^{T} \widehat{
  \bm \Sigma} \bm z$. Because $\widehat{\bm \Sigma}_k, \widehat{\bm \Sigma}\in
\mathbb{R}_{p \times p}^{\ge}$, we have $\bm z \in \mathcal{N}(\widehat{
  \bm \Sigma})$ and $\bm z \in \mathcal{N}(\widehat{\bm \Sigma}_k)$. In particular,
we have that $\mathcal{N}\{\widehat{\bm \Sigma}_k(\lambda)\} \subset
\mathcal{N}(\widehat{\bm \Sigma})$. Now, suppose $\bm z \in
\mathcal{N}(\widehat{\bm \Sigma})$. Similarly, we have that $0 = \bm z^{T}
\widehat{\bm \Sigma} \bm z = N^{-1} \sum_{k = 1}^K n_k \bm z^{T} \widehat{
  \bm \Sigma}_k \bm z$, which implies that $\bm z \in \mathcal{N}(\widehat{
  \bm \Sigma}_k)$ because $\widehat{\bm \Sigma}_k \in \mathbb{R}_{p \times
  p}^{\ge}$. Therefore, $\mathcal{N}(\widehat{\bm \Sigma}) \subset
\mathcal{N}(\widehat{\bm \Sigma}_k)$.
\end{proof}

In Lemma \ref{lemma:rda-tilde-Sigma_k} below, we derive an alternative
expression for $\tilde{\bm \Sigma}_k$ in terms of the matrix of eigenvectors of
$\widehat{\bm \Sigma}$. Let $\widehat{\bm \Sigma} = \bm U \bm D \bm U^{T}$ be
the eigendecomposition of $\widehat{\bm \Sigma}$ such that $\bm D \in
\mathbb{R}_{p \times p}^{\ge}$ is the diagonal matrix of eigenvalues of
$\widehat{\bm \Sigma}$ with
\begin{align*}
  \bm D = \begin{bmatrix}
    \bm D_q & \bm 0 \\
    \bm{0} & \bm 0_{p-q}
    \end{bmatrix},
\end{align*}
$\bm D_q \in \mathbb{R}_{q \times q}^{>}$ is the diagonal matrix consisting of
the positive eigenvalues of $\widehat{\bm \Sigma}$, the columns of $\bm U \in
\mathbb{R}_{p \times p}$ are the corresponding orthonormal eigenvectors of
$\widehat{\bm \Sigma}$, and rank$(\widehat{\bm \Sigma}) = q$. Then, we partition
$\bm U = (\bm U_1, \bm U_2)$ such that $\bm U_1 \in \mathbb{R}_{p \times q}$ and
$\bm U_2 \in \mathbb{R}_{p \times (p - q)}$.

\begin{lemma}\label{lemma:rda-tilde-Sigma_k}
Let $\widehat{\bm \Sigma} = \bm U \bm D \bm U^{T}$ be the eigendecomposition of
$\widehat{\bm \Sigma}$ as above, and suppose that rank$(\widehat{\bm \Sigma}) =
q \le p$. Then, we have
\begin{align}
	\tilde{\bm \Sigma}_k &= \bm U
    \begin{bmatrix}
      \bm W_k & \bm 0 \\
      \bm 0 & \gamma \bm I_{p-q}
    \end{bmatrix}
    \bm U^{T}, \quad k = 1, \ldots, K,\label{eq:rda-matrix}
\intertext{where}
\bm W_k &= \alpha_k \{(1 - \lambda) \bm U_1^{T} \widehat{\bm \Sigma}_k \bm U_1 + \lambda \bm D_q\} + \gamma \bm I_{q}.\label{eq:Wk}
\end{align}
\end{lemma}
\begin{proof}
From Lemma \ref{lemma:null-spaces}, the columns of $\bm U_2$ span the null space
of $\widehat{\bm \Sigma}_k$, which implies that $\widehat{\bm \Sigma}_k \bm U_2
= \bm 0_{p \times (p - q)}$. Hence,
\begin{align*}
  \bm U^{T} \widehat{\bm \Sigma}_k \bm U = \begin{bmatrix}
    \bm U_1^{T} \widehat{\bm \Sigma}_k \bm U_1 & \bm 0 \\
    \bm 0 & \bm 0_{p-q}
  \end{bmatrix}, \quad k = 1, \ldots, K.
\end{align*}
Thus, $\bm U^{T} \tilde{\bm \Sigma}_k \bm U = \alpha_k \{(1 - \lambda) \bm U^{T}
\widehat{\bm \Sigma}_k \bm U + \lambda \bm D\} + \gamma \bm I_p$, and
\eqref{eq:rda-matrix} holds because $\bm U$ is orthogonal.
\end{proof}

As an immediate consequence of Lemma \ref{lemma:rda-tilde-Sigma_k}, we have the
following corollary.

\begin{cor}
Let $\widehat{\bm \Sigma}_k(\lambda)$ be defined as in
\eqref{eq:sig-lambda-alternative}. Then, for $\lambda \in (0, 1]$,
  rank$\{\widehat{\bm \Sigma}_k(\lambda)\} = q$, $k = 1, \ldots, K$.
\end{cor}
\begin{proof}
The proof follows when we set $\gamma = 0$ in Lemma
\ref{lemma:rda-tilde-Sigma_k}.
\end{proof}

Thus, by incorporating each $\bm x_i$ into the estimation of $\bm \Sigma_k$, we
increase the rank of $\widehat{\bm \Sigma}_k(\lambda)$ to $q \approx N$ if
$\lambda \ne 0$. Next, we provide an essential result that enables us to prove
that \eqref{eq:hdrda} is invariant to adjustments to the eigenvalues of $\tilde{
  \bm \Sigma}_k$ corresponding to the null space of $\widehat{\bm \Sigma}$.

\begin{lemma}\label{lemma:RDA-constant-term}
Let $\bm U_2$ be defined as above. Then, for all $\bm x \in \mathbb{R}_{p \times
  1}$, $\bm U_2^{T} (\bm x - \xbar_k) = \bm U_2^{T} (\bm x - \xbar_{k'})$, $1
\le k, k' \le K$, where $k \ne k'$.
\end{lemma}
\begin{proof}
Let $\bm x \in \mathbb{R}_{p \times 1}$, and suppose that $1 \le k, k' \le K$.
Recall that $\bm U_2 \in \mathcal{N}(\widehat{\bm \Sigma})$, which implies that
$\bm U_2^{T} \in \mathcal{C}(\widehat{\bm \Sigma})^{\perp}$ \citep[Lemma
  1.2.5]{Kollo:2005vp}. Now, because $\bm x_i \in \mathcal{C}(\widehat{ \bm
  \Sigma})$ $(i = 1, \ldots, N)$, $\bm U_2^{T} \bm x_i = \bm 0_{p-q}$.  Hence,
$\bm 0_{p-q} = \sum_{i=1}^N \beta_i \bm U_2^{T}\bm x_i = \bm U_2^{T}(\xbar_k -
\xbar_{k'})$, where $\beta_i = (n_k n_{k'})^{-1} \{ I(y_i = k) n_{k'} - I(y_i =
k') n_k \}$. Therefore, $\bm U_2^{T} (\bm x - \xbar_k) = \bm U_2^{T} (\bm x -
\xbar_{k'})$.
\end{proof}

We now present our main result, where we decompose \eqref{eq:hdrda} and show
that the term requiring the largest computational costs does not contribute to
the classification of an unlabeled observation performed using Lemma
\ref{lemma:RDA-constant-term}. Hence, we reduce \eqref{eq:hdrda} to an
equivalent, more computationally efficient decision rule.

\begin{thm}
Let $\tilde{\bm \Sigma}_k$ and $\bm W_k$ be defined as in \eqref{eq:rda-matrix}
and \eqref{eq:Wk}, respectively, and let $\bm U_1$ be defined as above. Then,
the decision rule in \eqref{eq:hdrda} is equivalent to
\begin{align}
		D_{HDRDA}(\bm x) &= \argmin_k  (\bm x - \xbar_k)^{T} \bm U_1 \bm W_k^{-1} \bm U_1^{T} (\bm x - \xbar_k) + \log | \bm W_k |. \label{eq:hdrda-decomposed}
\end{align}
\end{thm}
\begin{proof}
From \eqref{eq:rda-matrix}, we have that
\begin{align*}
  \tilde{\bm \Sigma}_k^{+} = \bm U \begin{bmatrix}
    \bm W_k^{-1} & \bm 0 \\
    \bm 0 & \gamma^{+} \bm I_{p-q}
  \end{bmatrix}
  \bm U^{T}
\end{align*}
and $|\tilde{\bm \Sigma}_k| = \gamma^{p-q}| \bm W_k |$, $k = 1, \ldots,
K$. Therefore, for all $\bm x \in \mathbb{R}_{p \times 1}$, we have that
\begin{align*}
	(\bm x - \xbar_k)^{T} \tilde{\bm \Sigma}_k^{+}(\bm x - \xbar_k)  + \log |\tilde{\bm \Sigma}_k| &= (\bm x - \xbar_k)^{T} \bm U_1 \bm W_k^{-1} \bm U_1^{T} (\bm x - \xbar_k)\\
	&+ \gamma^{+} (\bm x - \xbar_k)^{T} \bm U_2 \bm U_2^{T} (\bm x - \xbar_k) + \log | \bm W_k |\\
  &+ (p - q) \log \gamma.
\end{align*}
Because $\gamma$ is constant for $k = 1, \ldots, K$, we can omit the $(p - q)
\log \gamma$ term and particularly avoid the calculation of $\log 0$ for $\gamma
= 0$. Then, the proof follows from Lemma \ref{lemma:RDA-constant-term} because
$\bm U_2^{T} (\bm x - \xbar_k)$ is constant for $k = 1, \ldots, K$.
\end{proof}

Using Theorem 1, we can avoid the time-consuming inverses and determinants of $p
\times p$ covariance matrices in \eqref{eq:hdrda} and instead calculate these
same operations on $\bm W_k \in \mathbb{R}_{q \times q}$ in
\eqref{eq:hdrda-decomposed}. The substantial computational improvements arise
because our proposed classifier in \eqref{eq:hdrda} is invariant to the term
$\bm U_2$, thus yielding an equivalent classifier in \eqref{eq:hdrda-decomposed}
with a substantial reduction in computational complexity. Here, we demonstrate
that the computational efficiency in calculating the inverse and determinant of
$\bm W_k$ can be further improved via standard matrix operations when we show
that the inverses and determinants of $\bm W_k$ can be performed on matrices of
size $n_k \times n_k$.

\begin{proposition}\label{proposition:hdrda-W_k}
Let $\bm W_k$ be defined as above. Then, $|\bm W_k| = |\bm \Gamma_k| |\bm Q_k|$
and
\begin{align}
		\bm W_k^{-1} &= \bm \Gamma_k^{-1} - n_k^{-1} \alpha_k(1 - \lambda) \bm \Gamma_k^{-1} \bm U_1^{T} \bm X_k^{T} \bm Q_k^{-1} \bm X_k \bm U_1 \bm \Gamma_k^{-1},\label{eq:W_k_inv}
    \intertext{where}
    \bm Q_k &= \bm I_{n_k} + n_k^{-1} \alpha_k(1 - \lambda) \bm X_k \bm U_1 \bm \Gamma_k^{-1} \bm U_1^{T} \bm X_k^{T}\label{eq:Q_k}
    \intertext{and}
    \bm \Gamma_k &= \alpha_k \lambda \bm D_q + \gamma \bm I_q\label{eq:Gamma_k}.
\end{align}
\end{proposition}
\begin{proof}
First, we write $\bm W_k = n_k^{-1} \alpha_k (1 - \lambda) \bm U_1^{T} \bm
X_k^{T} \bm X_k \bm U_1 + \bm \Gamma_k$. To calculate $|\bm W_k|$, we apply
Theorem 18.1.1 from \cite{Harville:2008wja}, which states that $|\bm A + \bm B
\bm T \bm C| = |\bm A| |\bm T| |\bm T^{-1} + \bm C \bm A^{-1} \bm B|$, where
$\bm A \in \mathbb{R}_{a \times a}^{>}$, $\bm B \in \mathbb{R}_{a \times b}$,
$\bm T \in \mathbb{R}_{b \times b}^{>}$, and $\bm C \in \mathbb{R}_{b \times
  a}$. Thus, setting $\bm A = \bm \Gamma_k$, $\bm B = \alpha_k (1 - \lambda) \bm
U_1^{T} \bm X_k^{T}$, $\bm T = \bm I_{n_k}$, and $\bm C = \bm X_k \bm U_1$, we
have $|\bm W_k| = |\bm \Gamma_k| |\bm Q_k|$. Similarly, \eqref{eq:W_k_inv}
follows from the well-known Sherman-Woodbury formula \citep[Theorem
  18.2.8]{Harville:2008wja} because $(\bm A + \bm B \bm T \bm C)^{-1} =
\bm A^{-1} - \bm A^{-1} \bm B (\bm T^{-1} + \bm C \bm A^{-1} \bm B)^{-1} \bm C
\bm A^{-1}$.
\end{proof}

Notice that $\bm \Gamma_k$ is singular when $(\lambda, \gamma) = (0, 0)$
because $\bm \Gamma_k = \bm 0_q$, in which case we use the formulation in
\eqref{eq:hdrda-decomposed} instead. Also, notice that if $\alpha_k$ is
constant across the $K$ classes, then $\bm \Gamma_k$ in \eqref{eq:Gamma_k} is
independent of $k$. Consequently, $|\bm \Gamma_k|$ is constant across the $K$
classes and need not be calculated in \eqref{eq:hdrda-decomposed}.

\subsection{Model Selection}

Thus far, we have presented the \emph{HDRDA} classifier and its properties that
facilitate an efficient calculation of the decision rule. Here, we describe an
efficient model-selection procedure along with pseudocode in Algorithm
\ref{alg:hdrda_model_selection} to select the optimal tuning-parameter estimates
from the Cartesian product of candidate values $\{\lambda_g\}_{g=1}^G \times
\{\gamma_h\}_{h=1}^H$. We estimate the $V$-fold cross-validation error rate for
each candidate pair and select $(\widehat{\lambda}, \widehat{\gamma})$, which
attains the minimum error rate. To calculate the $V$-fold cross-validation, we
partition the original training data into $V$ mutually exclusive and exhaustive
folds that have approximately the same number of observations. Then, for $v = 1,
\ldots, V$, we classify the observations in the $v$th fold by training a
classifier on the remaining $V - 1$ folds. We calculate the cross-validation
error as the proportion of misclassified observations across the $V$ folds.

\IncMargin{1em}
\begin{algorithm}
  \SetKwInOut{Input}{input}
  \SetKwInOut{Output}{output}

  \Input{
    Data matrix $\bm X$\\
    Parameter grid $\{\lambda_g\}_{g=1}^G \times \{\gamma_h\}_{h=1}^H$
  }
  \Output{Optimal Estimates $(\hat{\lambda}, \hat{\gamma})$}
  \BlankLine
  \For{$v \leftarrow 1$ \KwTo $V$} {
    Partition $\bm X$ into $\bm X_{train} \in \mathbb{R}_{N \times p}$ and
    $\bm X_{test} \in \mathbb{R}_{N_T \times p}$

    \For{$k\leftarrow 1$ \KwTo $K$}{
      Extract $\bm X_k \in \mathbb{R}_{n_k \times p}$ from $\bm X_{train}$

      Compute sample mean $\xbar_k$ from $\bm X_k$

      Center $\bm X_k \leftarrow \bm X_k - \bm 1_{n_k} \xbar_k^T$
    }

    $\bm X_c \leftarrow [\bm X_1^T, \ldots, \bm X_K^T]^T$

    Compute the compact SVD $\bm X_c = \bm M_q \bm D_q \bm U_1^T$

    Transform $\bm X_c \leftarrow \bm X_c \bm U_1$

    Transform $\bm X_{test} \leftarrow \bm X_{test} \bm U_1$

    \For{$k \leftarrow 1$ \KwTo $K$} {
      Extract $\bm X_k \in \mathbb{R}_{n_k \times q}$ from $\bm X_c$

      Recompute sample mean $\xbar_k$ from $\bm X_k$
    }

    \For{$(\lambda, \gamma) \in \{\lambda_g\}_{g=1}^G \times \{\gamma_h\}_{h=1}^H$} {
      \For{$k \leftarrow 1$ \KwTo $K$} {
        Compute $\bm Q_k$ using \eqref{eq:Q_k}

        Compute $\bm \Gamma_k$ using \eqref{eq:Gamma_k}

        Compute $\bm W_k^{-1}$ using \eqref{eq:W_k_inv}

        Compute $|\bm W_k| = |\bm \Gamma_k| |\bm Q_k|$

        Compute $(\bm x - \xbar_k)^{T} \bm U_1 \bm W_k^{-1} \bm U_1^{T} (\bm x -
        \xbar_k) + \log | \bm W_k |$ for each row $\bm x$ of $\bm X_{test}$
      }
      Classify test observations $\bm X_{test}$ using \eqref{eq:hdrda-decomposed}

      Compute the number of misclassified test observations $\#\{\text{Error}_v(\lambda, \gamma)\}$
    }
   }
  Compute $\widehat{\text{Error}}(\lambda, \gamma) = N^{-1} \sum_{v=1}^V \#\{\text{Error}_v(\lambda, \gamma)\}$

  Report optimal $(\hat{\lambda}, \hat{\gamma}) \leftarrow \argmin_{(\lambda, \gamma)} \widehat{\text{Error}}(\lambda, \gamma)$
  \caption{Model selection for the \emph{HDRDA} classifier}
  \label{alg:hdrda_model_selection}
\end{algorithm}
\DecMargin{1em}

A primary contributing factor to the efficiency of Algorithm
\ref{alg:hdrda_model_selection} is our usage of the compact singular value
decomposition (SVD). Rather than computing the eigenvalue decomposition of
$\widehat{\bm \Sigma}$ to obtain $\bm U_1$, we instead obtain $\bm U_1$ by
computing the eigendecomposition of a much smaller $N \times N$ matrix when $p
\gg N$ \cite[Chapter~18.3.5]{Hastie:2008dt}. Applying the SVD, we decompose $\bm
X_c = \bm M \bm \Delta \bm U^{T}$, where $\bm M \in \mathbb{R}_{N \times p}$ is
orthogonal, $\bm \Delta \in \mathbb{R}_{p \times p}^{\ge}$ is a diagonal matrix
consisting of the singular values of $\bm X_c$, and $\bm U \in \mathbb{R}_{p
  \times p}$ is orthogonal. Recalling that $\widehat{\bm \Sigma} = N^{-1} \bm
X_c^{T} \bm X_c$, we have the eigendecomposition $\widehat{\bm \Sigma} = \bm U
\bm D \bm U^{T}$, where $\bm U$ is the matrix of eigenvectors of $\widehat{\bm
  \Sigma}$ and $\bm D = N^{-1} \bm \Delta$ is the diagonal matrix of eigenvalues
of $\widehat{\bm \Sigma}$. Now, we can obtain $\bm M$ and $\bm D$ efficiently
from the eigenvalue decomposition of the $N \times N$ matrix $\bm X_c \bm
X_c^{T} = \bm M \bm D \bm M^{T}$. Next, we compute $\bm U = \bm X_c^{T} \bm M
\bm D^{+/2}$, where
\begin{align*}
  \bm D^{+/2} = \begin{bmatrix}
    \bm D_q^{-1/2} & \bm 0 \\
    \bm 0 & \bm 0_{N-q}
  \end{bmatrix}.
\end{align*}
We then determine $q$, the number of numerically nonzero eigenvalues present in
$\bm D$, by calculating the number of eigenvalues that exceeds some tolerance
value, say, $1 \times 10^{-6}$. We then extract $\bm U_1$ as the first $q$
columns of $\bm U$.

As a result of the compact SVD, we need calculate $\bm X_c \bm U_1$ only once
per cross-validation fold, requiring $O(p q N) \approx O(p N^2)$ calculations.
Hence, the computational costs of expensive calculations, such as matrix
inverses and determinants, are greatly reduced because they are performed in the
$q$-dimensional subspace. Similarly, we reduce the dimension of the test data
set by calculating $\bm X_{test} \bm U_1$ once per fold. Conveniently, we see
that the most costly computation involved in $\bm Q_k$ and $\bm W_k^{-1}$ is
$\bm X_k \bm U_1$, which can be extracted from $\bm X_c \bm U_1$. Thus, after
the initial calculation of $\bm X_c \bm U_1$ per cross-validation fold, $\bm
Q_k$ requires $O(n_k q^2)$ operations. Because $\bm Q_k \in \mathbb{R}_{n_k
  \times n_k}$, both its determinant and inverse require $O(n_k^3)$
operations. Consequently, $\bm W_k^{-1}$ requires $O(n_k q^2)$ operations. Also,
the inverse of the diagonal matrix $\bm \Gamma_k^{-1} \in \mathbb{R}_{q \times
  q}$ requires $O(q)$ operations. Finally, we remark that $| \bm W_k |$ requires
$O(n_k^3)$ operations.

The expressions given in Proposition \ref{proposition:hdrda-W_k} also expedite
the selection of $\lambda$ and $\gamma$ via cross-validation because the most
time-consuming matrix operation involved in computing $\bm W_k^{-1}$ and $|\bm
W_k|$ is $\bm X_k \bm U_1 \in \mathbb{R}_{n_k \times q}$, which is independent
of $\lambda$ and $\gamma$. The subsequent operations in calculating $\bm
W_k^{-1}$ and $|\bm W_k|$ can be simply updated for different pairs of $\lambda$
and $\gamma$ without repeating the costly computations. Also, rather than calculating
$(\bm x - \xbar_k)^{T} \bm U_1 \bm W_k^{-1} \bm U_1^{T} (\bm x - \xbar_k)$
individually for each row $\bm x$ of $\bm X_{test}$, we can calculate $(\bm
X_{test} - \xbar_k \bm 1_k')' \bm U_1 \bm W_k^{-1} \bm U_1^{T} (\bm X_{test} -
\xbar_k \bm 1_k')$. The diagonal elements of the resulting matrix contain the
individual quadratic form of each test observation, $\bm x_t$.

\section{Timing Comparisons between RDA and HDRDA}
\label{sec:timing-comparisons}

In this section, we demonstrate that the computational performance of the model
selection employed in the \emph{HDRDA} classifier is substantially faster than
that of the \emph{RDA} classifier on small-sample, high-dimensional data
sets. The relative difference in runtime between the two classifiers drastically
increases as $p$ increases. To compare the two classifiers, we generated 25
observations from each of $K=4$ multivariate normal populations with mean
vectors $\bm \mu_1 = -3 \cdot \bm 1_p$, $\bm \mu_2 = -\bm 1_p$, $\bm \mu_3 = \bm
1_p$, and $\bm \mu_4 = 3 \cdot \bm 1_p$. We set the covariance matrix of each
population to the $p \times p$ identity matrix. For each data set generated, we
estimated the parameters $\lambda$ and $\gamma$ for both classifiers using a
grid of 5 equidistant candidate values between 0 and 1, inclusively. We set
$\alpha_k = 1 - \gamma$, $k = 1, \ldots, K$, in the \emph{HDRDA} classifier. At
each pair of $\lambda$ and $\gamma$, we computed the 10-fold cross-validation
error rate \citep{Hastie:2008dt}. Then, we selected the model that minimized the
10-fold cross-validation error rate.

We compared the runtime of both classifiers by increasing the number of features
from $p = 500$ to $p = 5000$ in increments of 500. Next, we generated 100 data
sets for each value of $p$ and computed the training and model selection runtime
of both classifiers. Our timing comparisons are based on our \emph{HDRDA}
implementation in the {\tt sparsediscrim} R package and the standard \emph{RDA}
implementation in the {\tt klaR} R package. All timing comparisons were
conducted on an Amazon Elastic Compute Cloud (EC2) {\tt c4.4xlarge} instance
using version 3.3.1 of the open-source statistical software {\tt R}. Our timing
comparisons can be reproduced with the code available at
\url{https://github.com/ramhiser/paper-hdrda}.

\subsection{Timing Comparison Results}

In Figure \ref{fig:timing-results}, we plotted the runtime of the model
selections for both the \emph{HDRDA} and \emph{RDA} classifiers as a function of
$p$. We observed that the \emph{HDRDA} classifier was substantially faster than
the \emph{RDA} classifier as $p$ increased. In the left panel of Figure
\ref{fig:timing-results}, we fit a quadratic regression line to the \emph{RDA}
runtimes and a simple linear regression model to the \emph{HDRDA} runtimes.  For
improved understanding, in the right panel we repeated the same scatterplot and
linear fit with the timings restricted to the observed range of the \emph{HDRDA}
timings. Figure \ref{fig:timing-results} suggests that the usage of a matrix
inverse and determinant in the {\tt klaR} R package's discriminant function
yielded model-selection timings that exceeded linear growth in $p$. Because the
\emph{HDRDA} classifier removes inverse and determinants, it was computationally
more efficient than the \emph{RDA} classifier, especially as $p$ increased. In
fact, when $p = 5000$, the \emph{RDA} classifier required
24.933 minutes on average to perform model selection, while
the \emph{HDRDA} classifier selected its optimal model in
2.979 seconds on average. Clearly, the model selection
employed by the \emph{HDRDA} classifier is substantially faster than that of the
\emph{RDA} classifier.

We quantified the relative timing comparisons between the two classifiers by
calculating the ratio of mean timings of the \emph{RDA} classifier to the
\emph{HDRDA} classifier for each value of $p$. We employed nonparametric
bootstrapping to estimate the mean ratio along with 95\% confidence
intervals. In Figure \ref{fig:timing-comparison-bootstrap}, the bootstrap
sampling distributions for the ratio of mean timings are given. First, we
observe that the mean relative timings increased as $p$ increased. For smaller
dimensions, the relative difference in computing was sizable with the average
ratio of the mean timings equal to 14.513 for
$p = 500$ and a 95\% confidence interval of
(14.191,
14.855). Furthermore, the ratio of mean
computing times suggested that the \emph{RDA} classifier is impractical for
higher dimensions. For instance, when $p = 5000$, the ratio of mean computing
times increased to 502.786 with a 95\%
confidence interval of (462.863,
546.396).

\[ \left[\text{Insert Figure \ref{fig:timing-results} approximately here }\right] \]

\[ \left[\text{Insert Figure \ref{fig:timing-comparison-bootstrap} approximately here }\right] \]

\section{Classification Study}
\label{sec:sims}

In this section, we compare our proposed classifier with four classifiers
recently proposed for small-sample, high-dimensional data along with the
random-forest classifier from \cite{Breiman:2001fb} using version 3.3.1 of the
open-source statistical software {\tt R}. Within our study, we included
penalized linear discriminant analysis from \cite{Witten:2011kc}, implemented in
the {\tt penalizedLDA} package. We also considered shrunken centroids
regularized discriminant analysis from \cite{Guo:2007te} in the {\tt rda}
package. Because the {\tt rda} package does not perform the authors' ``Min-Min''
rule automatically, we applied this rule within our {\tt R} code. We included
two modifications of diagonal linear discriminant analysis from
\cite{Tong:2012hw} and \cite{Pang:2009ik}, where the former employs an improved
mean estimator and the latter utilizes an improved variance estimator. Both
classifiers are available in the {\tt sparsediscrim} package. Finally, we
incorporated the random forest as a benchmark based on the findings of
\cite{FernandezDelgado:2014ul}, who concluded that the random forest is often
superior to other classifiers in benchmark studies. We used the implementation
of the random-forest classifier from the {\tt randomForest} package with 250
trees and 100 maximum nodes. For each classifer we explicitly set prior
probabilities as equal, if applicable. All other classifier options were set to
their default settings. Below, we refer to each classifier by the first author's
surname. All simulations were conducted on an Amazon EC2 {\tt c4.4xlarge}
instance. Our analyses can be reproduced via the code available at
\url{https://github.com/ramhiser/paper-hdrda}.

For the \emph{HDRDA} classifier in \eqref{eq:hdrda-decomposed}, we examined the
classification performance of two models. For the first \emph{HDRDA} model, we
set $\alpha_k = 1$, $k = 1, \ldots, K$, so that the covariance-matrix estimator
\eqref{eq:hdrda-cov} resembled \eqref{eq:ridge-estimator}. We estimated
$\lambda$ from a grid of 21 equidistant candidate values between 0 and 1,
inclusively. Similarly, we estimated $\gamma$ from a grid consisting of the
values $10^{-1}, \ldots, 10^4$, and $10^5$. We selected optimal estimates of
$\lambda$ and $\gamma$ using $10$-fold cross-validation. For the second model,
we set $\alpha_k = 1 - \gamma$, $k = 1, \ldots, K$, to resemble
\citeauthor{Friedman:1989tm}'s parameterization, and we estimated both $\lambda$
and $\gamma$ from a grid of 21 equidistant candidate values between 0 and 1,
inclusively.

We did not include the \emph{RDA} classifier in our classification study because
its training runtime was prohibitively slow on high-dimensional data in our
preliminary experiments. As shown in Section \ref{sec:timing-comparisons}, the
runtime of the \emph{RDA} classifier was drastically larger than that of the
\emph{HDRDA} classifier for a tuning grid of size $25 = 5 \times
5$. Consequently, a fair comparison between the \emph{RDA} and \emph{HDRDA}
classifiers would require model selection of $441 = 21 \times 21$ different
pairs of tuning parameters in the \emph{RDA} classifier. A tuning grid of this
size yielded excessively slow training runtimes for the \emph{RDA}
implementation from the {\tt klaR} R package.

\subsection{Simulation Study}

In this section we compare the competing classifiers using the simulation design
from \cite{Guo:2007te}. This design is widely used within the high-dimensional
classification literature, including the studies by \cite{Ramey:2013ji} and
\cite{Witten:2011kc}. First, we consider the block-diagonal covariance matrix
from \cite{Guo:2007te},
\begin{align}
  \bm\Sigma_k = \begin{bmatrix}
    \bm\Sigma^{(\rho_k)} & \bm 0_{100} & \bm 0_{100} & \cdots & \cdots & \cdots \\
    \bm 0_{100} & \bm\Sigma^{(-\rho_k)} & \bm 0_{100} & \bm 0_{100} & \cdots & \vdots \\
    \bm 0_{100} & \bm 0_{100} & \bm\Sigma^{(\rho_k)} & \bm 0_{100} & \cdots & \vdots \\
    \vdots & \bm 0_{100} & \bm 0_{100} & \bm\Sigma^{(-\rho_k)} & \bm 0_{100} & \vdots \\
    \vdots & \vdots & \vdots & \bm 0_{100} & \ddots & \vdots \\
    \cdots & \cdots & \cdots & \cdots & \cdots & \cdots \\
  \end{bmatrix},\label{eq:block-sigma}
\end{align}
where the $(i,j)$th entry of the block matrix $\bm\Sigma^{(\rho_k)} \in
\mathbb{R}_{100 \times 100}$ is
\begin{align*}
\bm\Sigma_{ij}^{(\rho_k)} = \{ \rho_k^{|i - j|} \}_{1 \le i,j \le 100}.
\end{align*}
The block-diagonal covariance structure in \eqref{eq:block-sigma} resembles
gene-expression data: within each block of pathways, genes are correlated, and
the correlation decays as a function of the distance between any two genes. The
original design from \cite{Guo:2007te} comprised two $p$-dimensional
multivariate normal populations with a common block-diagonal covariance matrix.

Although the design is indeed standard, the simulation configuration lacks
artifacts commonly observed in real data, such as skewness and extreme
outliers. As a result, we wished to investigate the effect of outliers on the
high-dimensional classifiers. To accomplish this goal, we generalized the
block-diagonal simulation configuration by sampling from a $p$-dimensional
multivariate contaminated normal distribution.  Denoting the PDF of the
$p$-dimensional multivariate normal distribution by $N_p(\bm x | \bm \mu, \bm
\Sigma)$, we write the PDF of the $k$th class as
\begin{align}
  p(\bm x | \omega_k) = (1 - \epsilon) N_p(\bm x | \bm \mu_k, \bm \Sigma_k) + \epsilon N_p(\bm x | \bm \mu_k, \eta \bm \Sigma_k),\label{eq:contaminated-pdf}
\end{align}
where $\epsilon \in [0, 1]$ is the probability that an observation is
contaminated (i.e., drawn from a distribution with larger variance) and $\eta >
1$ scales the covariance matrix $\bm \Sigma_k$ to increase the extremity of
outliers. For $\epsilon = 0$, we have the benchmark block-diagonal simulation
design from \cite{Guo:2007te}. As $\epsilon$ is increased, the average number of
outliers is increased. In our simulation, we let $\eta = 100$ and considered the
values of $\epsilon = 0$, $0.05$, $\ldots$, $0.50$.

We generated $K=3$ populations from \eqref{eq:contaminated-pdf} with $\bm
\Sigma_k$ given in \eqref{eq:block-sigma} and set the mean vector of class 1 to
$\bm \mu_1 = \bm 0_p$. Next, comparable to \cite{Guo:2007te}, the first 100
features of $\bm \mu_2$ were set to 1/2, while the rest were set to 0, i.e.,
$\bm \mu_2 = (\underbrace{1/2, \ldots, 1/2}_{100}, \underbrace{0, \ldots, 0}_{p
  - 100})$. For simplicity, we defined $\bm \mu_3 = -\bm \mu_2$. The three
populations differed in their mean vectors in the first 100 features
corresponding to the first block, and no difference in the means occurred in the
remaining blocks.

From each of the $K=3$ populations, we sampled 25 training observations ($n_k =
25$ for all $k$) and 10,000 test observations. After training each classifier on
the training data, we classified the test data sets and computed the proportion
of mislabeled test observations to estimate the classification error rate for
each classifier. Repeating this process 500 times, we computed the average of
the error-rate estimates for each classifier. We allowed the number of features
to vary from $p = 100$ to $p = 500$ in increments of 100 to examine the
classification accuracy as the feature dimension increased while maintaining a
small sample size. \cite{Guo:2007te} originally considered $\rho_k = 0.9$ for
all $k$. Alternatively, to explore the more realistic assumption of unequal
covariance matrices, we put $\rho_1 = 0.1$, $\rho_2 = 0.5$, and $\rho_3 = 0.9$.

\subsubsection{Simulation Results}

In Figure \ref{fig:sim-results}, we observed each classifier's average
classification error rates for the values of $\epsilon$ and $p$. Unsurprisingly,
the average error rate increased for each classifier as the contamination
probability $\epsilon$ increased regardless of the value of $p$. Sensitivity to
the presence of outliers was most apparent for the Pang, Tong, and Witten
classifiers. For smaller dimensions, the random-forest and \emph{HDRDA}
classifiers tended to outperform the remaining classifiers with the random
forest performing best. As the feature dimension increased with $p \ge 300$,
both \emph{HDRDA} classifiers outperformed all other classifiers, suggesting
that their inherent dimension reduction better captured the classificatory
information in the small training samples, even in the presence of outliers.

\[ \left[\text{Insert Figure \ref{fig:sim-results} approximately here }\right] \]

The Pang, Tong, and Witten methods yielded practically the same and consistently
the worst error rates when outliers were present with $\epsilon > 0$, suggesting
that these classifiers were sensitive to outliers. Notice, for example, that
when $p=400$, the error rates of the Pang, Tong, and Witten classifiers
increased dramatically from approximately 19\% when no outliers were present to
approximately 43\% when $\epsilon=0.05$. The sharp increase in average error
rates for these three classifiers continued as $\epsilon$ increased. Guo's
method always outperformed those of Pang, Witten, and Tong, but after outliers
were introduced, the Guo classifier's average error rate was not competitive
with the \emph{HDRDA} classifiers or the random-forest classifier.

\[ \left[\text{Insert Figure \ref{fig:sim-results-by-p} approximately here }\right] \]

In Figure \ref{fig:sim-results-by-p}, we again examine the simulation results as
a function of $p$ for a subset of the values of $\epsilon$. This set of plots
allows us to investigate the effect of feature dimensionality on classification
performance. When no outliers were present (i.e., $\epsilon=0$), the
random-forest classifier was outperformed by all other classifiers. Furthermore,
the \emph{HDRDA} classifiers were superior in terms of average error rate in
this setting. As $p$ increased, an elevation in average error rate was expected
for all classifiers, but the increase was not observed to be substantial.

For $\epsilon > 0$, we observed a different behavior in classification
performance. First, the Pang, Tong, and Witten methods, along with the
random-forest method, increased in average error rate as $p$ increased.
Contrarily, the performance of the \emph{HDRDA} and Guo classifiers was hardly
affected by $p$. Also, as discussed above, the \emph{HDRDA} classifiers were
superior to all other classifiers for large values of $p$ with only the
random-forest classifier outperforming them in smaller feature-dimension cases.

\subsection{Application to Gene Expression Data}

We compared the \emph{HDRDA} classifier to the five competing classifiers on six
benchmark gene-expression microarray data sets. First, we evaluated the
classification accuracy of each classifier by randomly partitioning the data set
under consideration such that $2/3$ of the observations were allocated as
training data and the remaining $1/3$ of the observations were allocated as a
test data set. To expedite the computational runtime, we reduced the training
data to the top 1000 variables by employing the variable-selection method
proposed by \cite{Dudoit:2002ev}. We then reduced the test data set to the same
1000 variables. After training each classifier on the training data, we
classified the test data sets and computed the proportion of mislabeled test
observations to estimate the classification error rate for each classifier.
Repeating this process 100 times, we computed the average of the error-rate
estimates for each classifier. We next provide a concise description of each
high-dimensional data set examined in our classification study.

\subsubsection{\cite{Chiaretti:2004gq} Data Set}

\cite{Chiaretti:2004gq} measured the gene-expression profiles for 128
individuals with acute lymphoblastic leukemia (ALL) using Affymetrix human 95Av2
arrays. Following \cite{Xu:2009fl}, we restricted the data set to $K = 2$
classes such that $n_1 = 74$ observations were without cytogenetic abnormalities
and $n_2 = 37$ observations had a detected BCR/ABL gene. The robust multichip
average normalization method was applied to all 12,625 gene-expression levels.

\subsubsection{\cite{Chowdary:2006kf} Data Set}

\cite{Chowdary:2006kf} investigated 52 matched pairs of tissues from colon and
breast tumors using Affymetrix U133A arrays and ribonucleic-acid (RNA)
amplification. Each tissue pair was gathered from the same patient and consisted
of a snap-frozen tissue and a tissue suspended in an RNAlater preservative.
Overall, 31 breast-cancer and 21 colon-cancer pairs were gathered, resulting in
$K = 2$ classes with $n_1 = 62$ and $n_2 = 42$. A purpose of the study was to
determine whether the disease state could be identified using 22,283
gene-expression profiles.

\subsubsection{\cite{Nakayama:2007fl} Data Set}

\cite{Nakayama:2007fl} acquired 105 gene-expression samples of 10 types of
soft-tissue tumors through an oligonucleotide microarray, including 16 samples
of synovial sarcoma (SS), 19 samples of myxoid/round cell liposarcoma (MLS), 3
samples of lipoma, 3 samples of well-differentiated liposarcoma (WDLS), 15
samples of dedifferentiated liposarcoma (DDLS), 15 samples of myxofibrosarcoma
(MFS), 6 samples of leiomyosarcoma (LMS), 3 samples of malignant nerve sheathe
tumor (MPNST), 4 samples of fibrosarcoma (FS), and 21 samples of malignant
fibrous histiocytoma (MFH). \cite{Nakayama:2007fl} determined from their data
that these 10 types fell into 4 broader groups: (1) SS; (2) MLS; (3) Lipoma,
WDLS, and part of DDLS; (4) Spindle cell and pleomorophic sarcomas including
DDLS, MFS, LMS, MPNST, FS, and MFH. Following \cite{Witten:2011kc}, we
restricted our analysis to the five tumor types having at least 15 observations.

\subsubsection{\cite{Shipp:2002ka} Data Set}

According to \cite{Shipp:2002ka}, approximately 30\%-40\% of adult non-Hodgkin
lymphomas are diffuse large B-cell lymphomas (DLBCLs). However, only a small
proportion of DLBCL patients are cured with modern chemotherapeutic regimens.
Several models have been proposed, such as the International Prognostic Index
(IPI), to determine a patient's curability. These models rely on clinical
covariates, such as age, to determine if the patient can be cured, and the
models are often ineffective. \cite{Shipp:2002ka} have argued that researchers
need more effective means to determine a patient's curability. The authors
measured 6,817 gene-expression levels from 58 DLBCL patient samples with
customized cDNA (lymphochip) microarrays to investigate the curability of
patients treated with cyclophosphamide, adriamycin, vincristine, and prednisone
(CHOP)-based chemotherapy. Among the 58 DLBCL patient samples, 32 are from cured
patients while 26 are from patients with fatal or refractory disease.

\subsubsection{\cite{Singh:2002fh} Data Set}

\cite{Singh:2002fh} have examined 235 radical prostatectomy specimens from
surgery patients between 1995 and 1997. The authors used oligonucleotide
microarrays containing probes for approximately 12,600 genes and expressed
sequence tags. They have reported that 102 of the radical prostatectomy
specimens are of high quality: 52 prostate tumor samples and 50 non-tumor
prostate samples.

\subsubsection{\cite{Tian:2003ht} Data Set}

\cite{Tian:2003ht} investigated the purified plasma cells from the bone marrow
of control patients along with patients with newly diagnosed multiple
myeloma. Expression profiles for 12,2625 genes were obtained via Affymetrix
U95Av2 microarrays. The plasma cells were subjected to biochemical and
immunohistochemical analyses to identify molecular determinants of osteolytic
lesions. For 36 multiple-myloma patients, focal bone lesions could not be
detected by magnetic resonance imaging (MRI), whereas MRI was used to detect
such lesions in 137 patients.

\subsubsection{Classification Results}

Similar to \cite{Witten:2011kc}, we report the average test error rates obtained
over 100 random training-test partitions in Table \ref{tab:microarray-results}
along with standard deviations of the test error rates in parentheses. The
\emph{HDRDA} and Guo classifiers were superior in classification performance for
the majority of the simulations. The \emph{HDRDA} classifiers yielded the best
classification accuracy on the Chowdary and Shipp data sets. Although the random
forest's accuracy slightly exceeded the \emph{HDRDA} classifiers on the Tian
data set, our proposed classifiers outperformed the other competing classifiers
considered here. Moreover, the \emph{HDRDA} classifiers yielded comparable
performance on five of the six data sets.

\[ \left[\text{Insert Table \ref{tab:microarray-results} approximately here }\right] \]

The average error-rate estimates for the Pang, Tong, and Witten classifiers were
comparable across all six data sets. Furthermore, the average error rates for
the Pang and Tong classifiers were approximately equal for all data sets except
for the Chiaretti dataset. This result suggests that the mean and variance
estimators used in lieu of the MLEs provided little improvement to
classification accuracies. However, we investigated the Pang classifier's poor
performance on the Chiaretti data set and determined that its variance estimator
exhibited numerical instability. The classifier's denominator was approximately
zero for both classes and led to the poor classification performance.

The random-forest classifier was competitive when applied to the Chowdary and
Singh data sets and yielded the smallest error rate of the considered
classifiers on the Tian data set. The fact that the \emph{HDRDA} and Guo
classifiers typically outperformed the random-forest classifier challenges the
claim of \cite{FernandezDelgado:2014ul} that random forests are typically
superior. Further studies should be performed to validate this statement in the
small-sample, high-dimensional setting.

Finally, the Pang, Tong, and Witten classifiers consistently yielded the largest
average error rates across the six data sets. Given that the standard deviations
were relatively large, we hesitate to generalize claims regarding the ranking of
these three classifiers in terms of the average error rate. However, the
classifiers' error rates and their variability across multiple random partitions
of each data set were large enough that we might question their benefit when
applied to real data.

\section{Discussion}
\label{sec:discussion}

We have demonstrated that our proposed \emph{HDRDA} classifier is competitive
with and often superior to random forests as well as the Witten, Pang, Tong, and
Guo classifiers. In fact, we have shown that the \emph{HDRDA} classifier often
yields superior classification accuracy when applied to small-sample,
high-dimensional data sets, confirming the assertions of \cite{Mai:2012bf} and
\cite{Fan:2012iq} that diagonal classifiers often yield inferior classification
performance when compared to other classification methods. Furthermore, we have
demonstrated that \emph{HDRDA} classifiers are more robust to the presence of
outliers than the diagonal classifiers despite their rapid computational
performance and their reduction in the number of parameters to estimate.

We also considered the popular penalized linear discriminant analysis from
\cite{Witten:2011kc} because it was specifically designed for high-dimensional
gene-expression data. We had expected its classification performance to be
competitive within our classification study and perhaps superior. Contrarily,
our empirical studies suggest that the classifier is sensitive to outliers and
unable to achieve comparable results with other classifiers designed for
small-sample, high-dimensional data. Also, despite the claims of
\cite{FernandezDelgado:2014ul} that random forests are typically superior to
other classifiers, we observed that they were indeed competitive but were
typically outperformed by classifiers developed for small-sample,
high-dimensional data.

We demonstrated that our \emph{HDRDA} implementation in the {\tt sparsediscrim}
R package can be used in practice with high-dimensional data sets. In our timing
comparisons, we showed that \emph{HDRDA} model selection could be employed on
data sets with $p = 5000$ in 2.979 seconds on
average. Contrarily, the \emph{RDA} classifier implemented in the {\tt klaR} R
package required 24.933 minutes on average to perform model
selection on data sets with $p = 5000$. Given that the \emph{RDA} classifier has
been shown to have excellent performance in the high-dimensional setting
\citep{Webb:2011vu} but is limited by its computationally intense
model-selection procedure, our work replaces the \emph{RDA} classifier for
high-dimensional data in practice. This result is reassuring because the
\emph{RDA} classifier remains widely popular in the literature. In fact,
variants of the RDA classifier have been applied to microarray data
\citep{Ching:2012fu,Li:2012ev,Tai:2007bk,Guo:2007te}, facial recognition
\citep{Zhang:2010va,Dai:2007vd,Lu:2005hq,Pima:2004vw,Lu:2003we}, handwritten
digit recognition \citep{Bouveyron:2007gx}, remote sensing
\citep{Tadjudin:1999fk}, seismic detection \citep{Anderson:2002kg}, and chemical
spectra \citep{Wu:1996us,Aeberhard:1993fp}.

The dimension reduction employed in this paper has reduced the dimension to
rank$(\widehat{\bm \Sigma}) = q$. An interesting extension of our work would
reduce the dimension $q$ further to a lower dimension $q_L < q$, perhaps using a
criterion similar to that of principal components analysis. While unclear
whether the classification performance would improve via such a method, the
efficiency of the model selection would certainly improve. Moreover, if $q_L =
2$ or $3$, low-dimensional graphical displays of high-dimensional data could be
obtained.

We thank Mrs. Joy Young for her numerous recommendations that enhanced the
quality of our writing.

\bibliographystyle{elsarticle-harv}
\bibliography{rda}

\clearpage

\begin{figure}
\begin{knitrout}
\definecolor{shadecolor}{rgb}{0.969, 0.969, 0.969}\color{fgcolor}
\includegraphics[width=.33\linewidth]{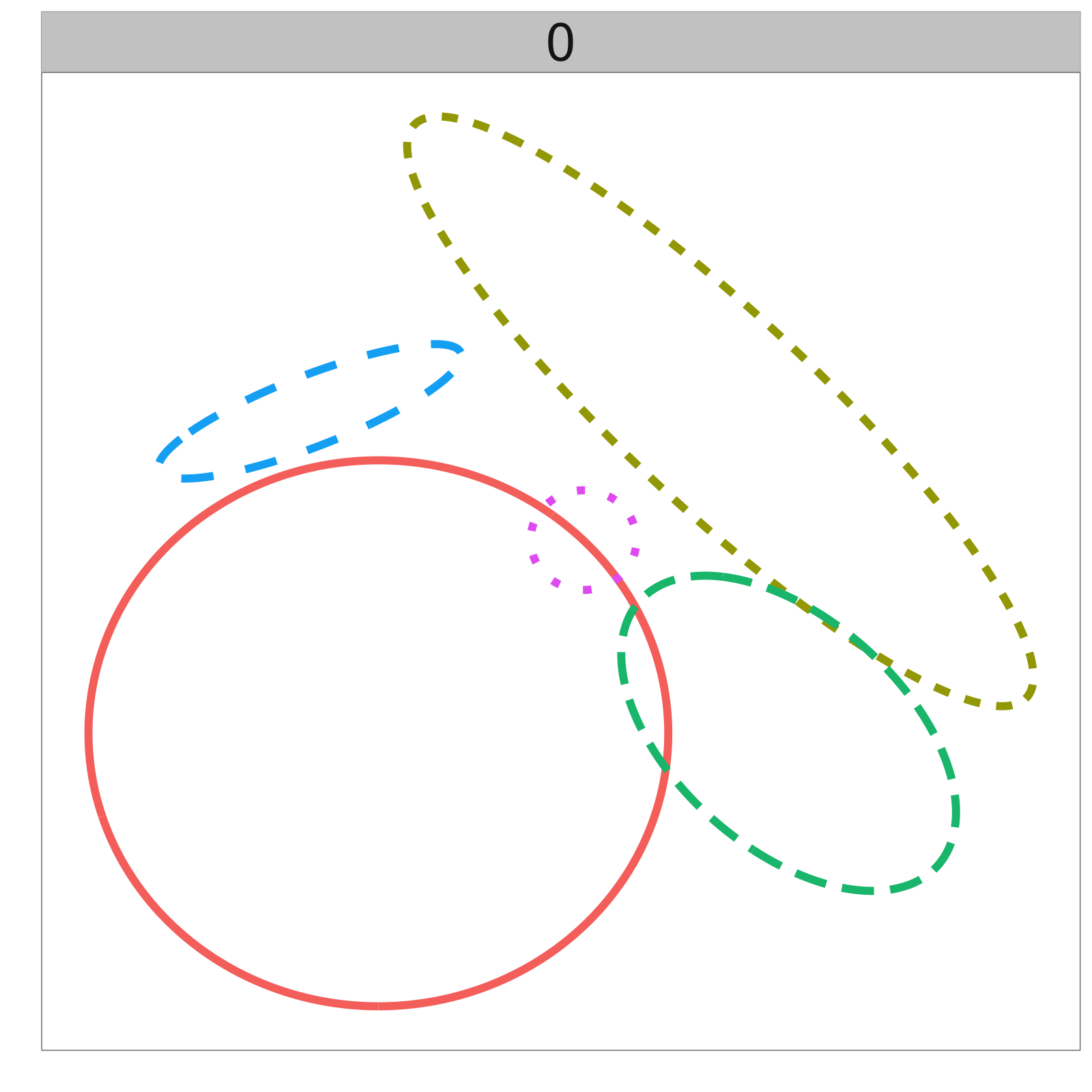} 
\includegraphics[width=.33\linewidth]{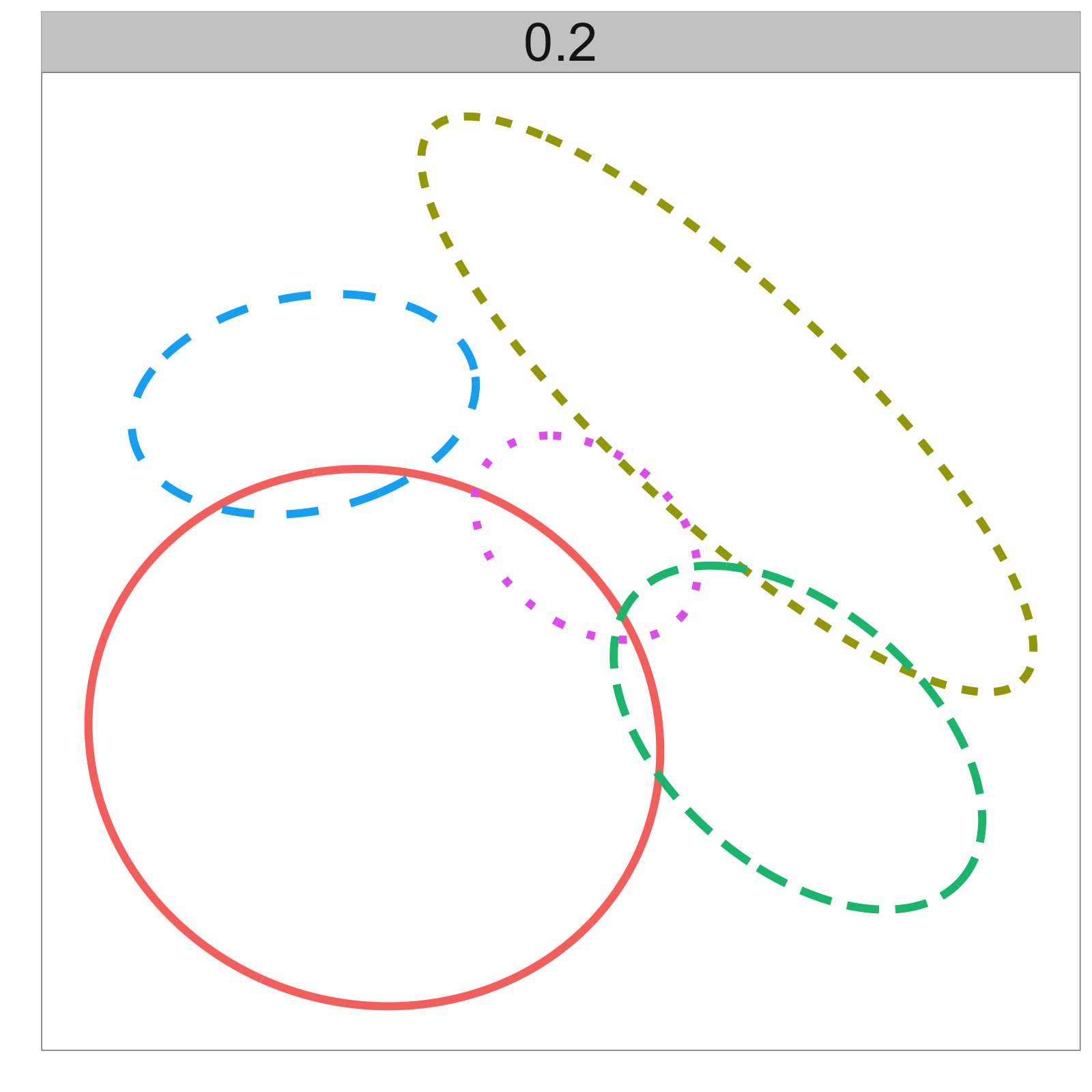} 
\includegraphics[width=.33\linewidth]{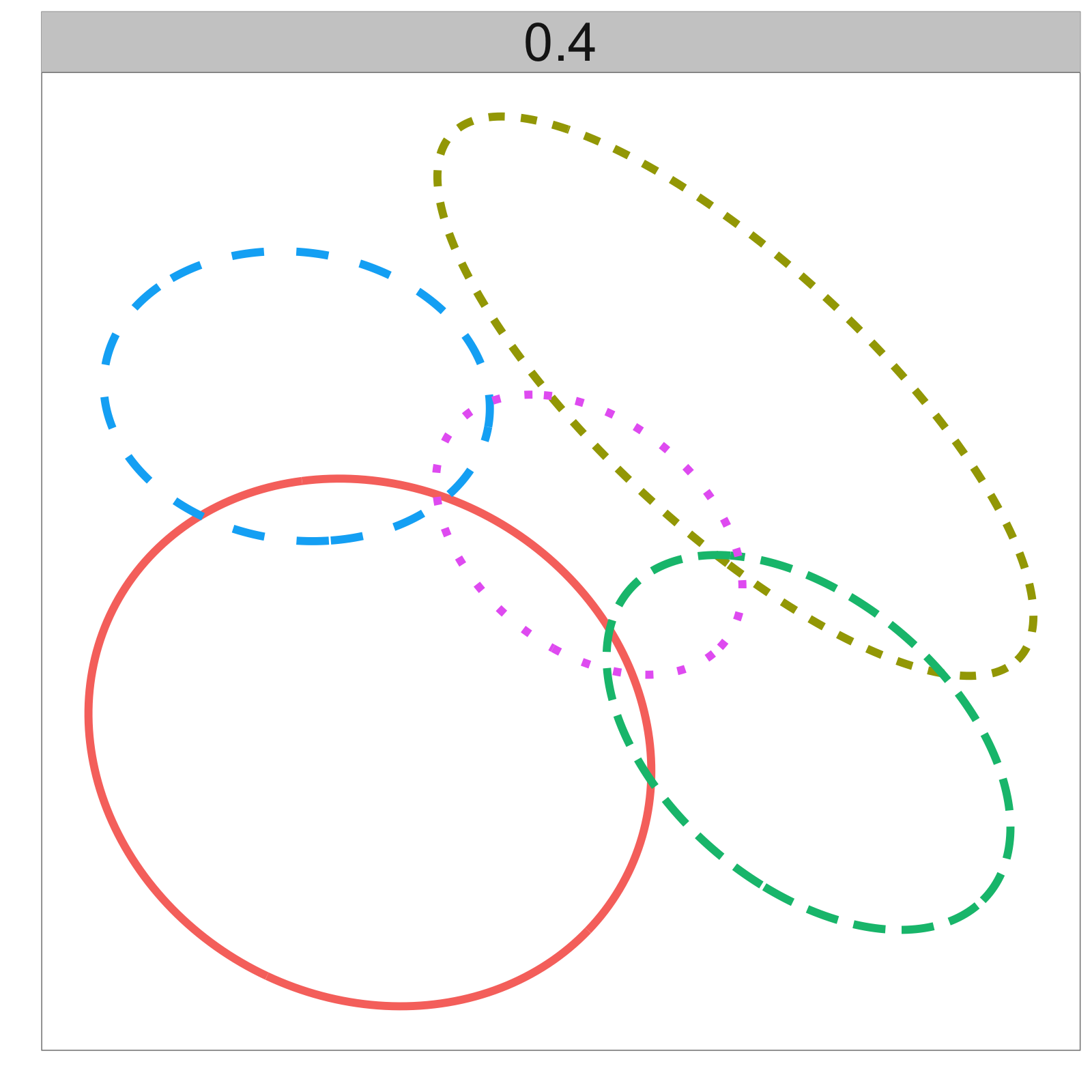} 
\includegraphics[width=.33\linewidth]{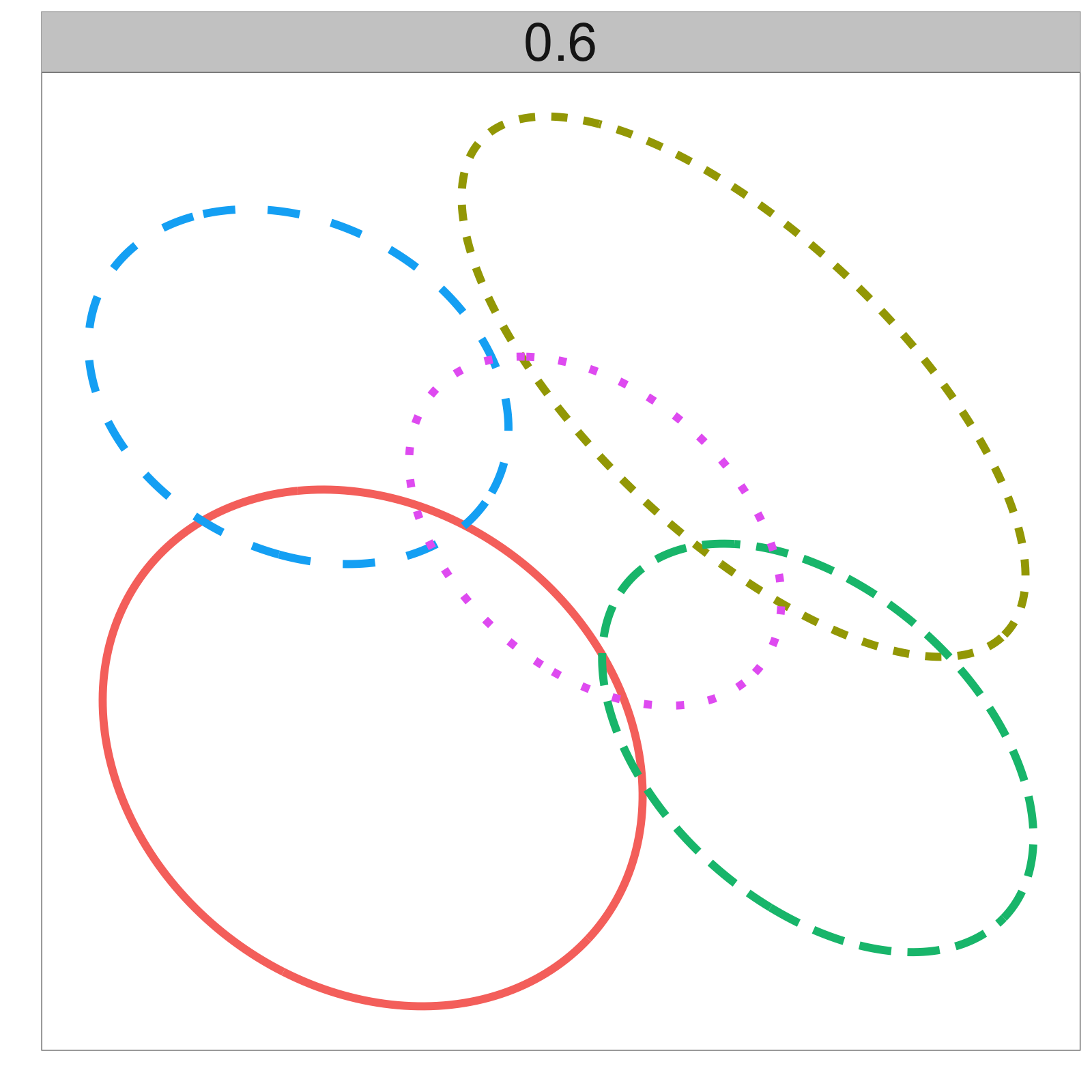} 
\includegraphics[width=.33\linewidth]{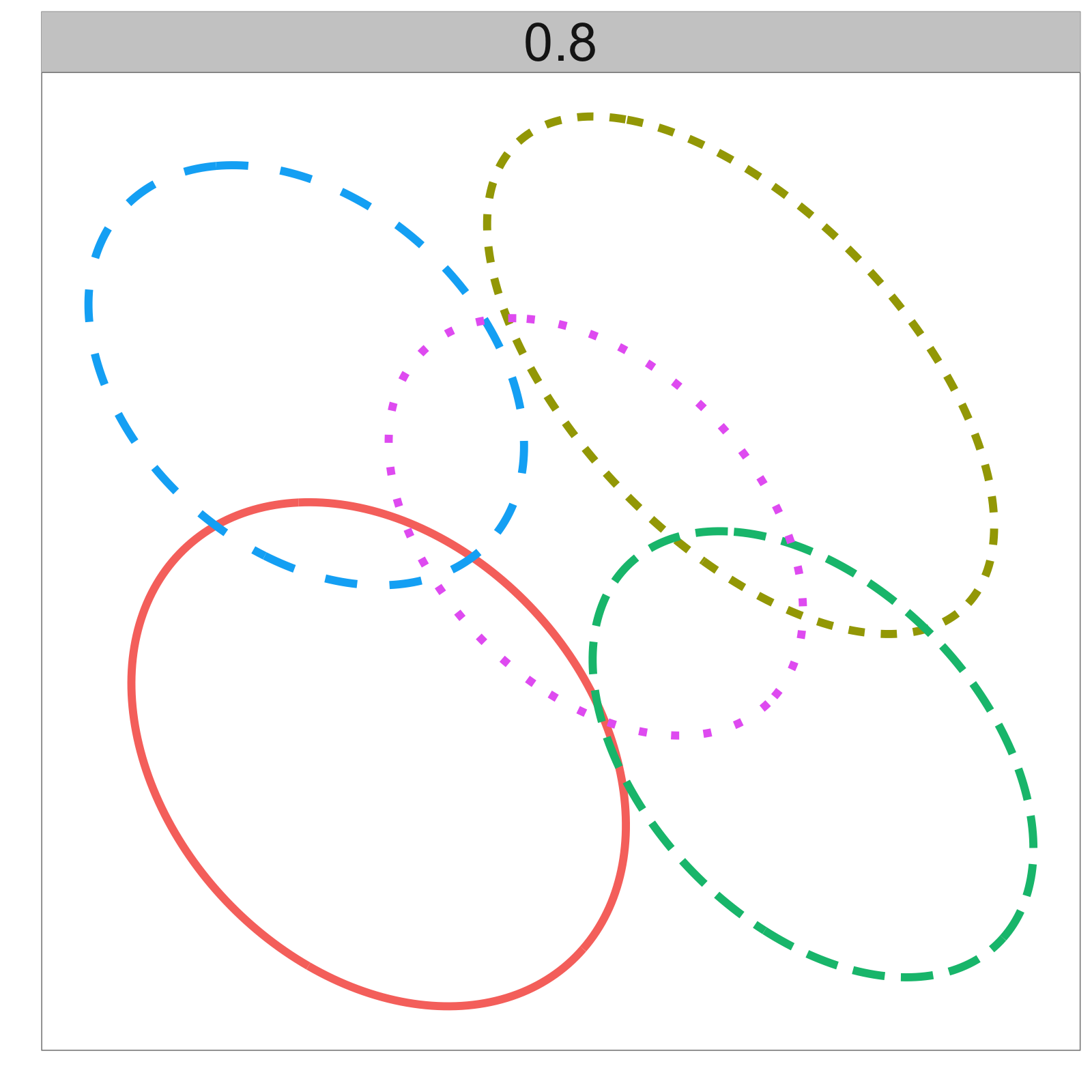} 
\includegraphics[width=.33\linewidth]{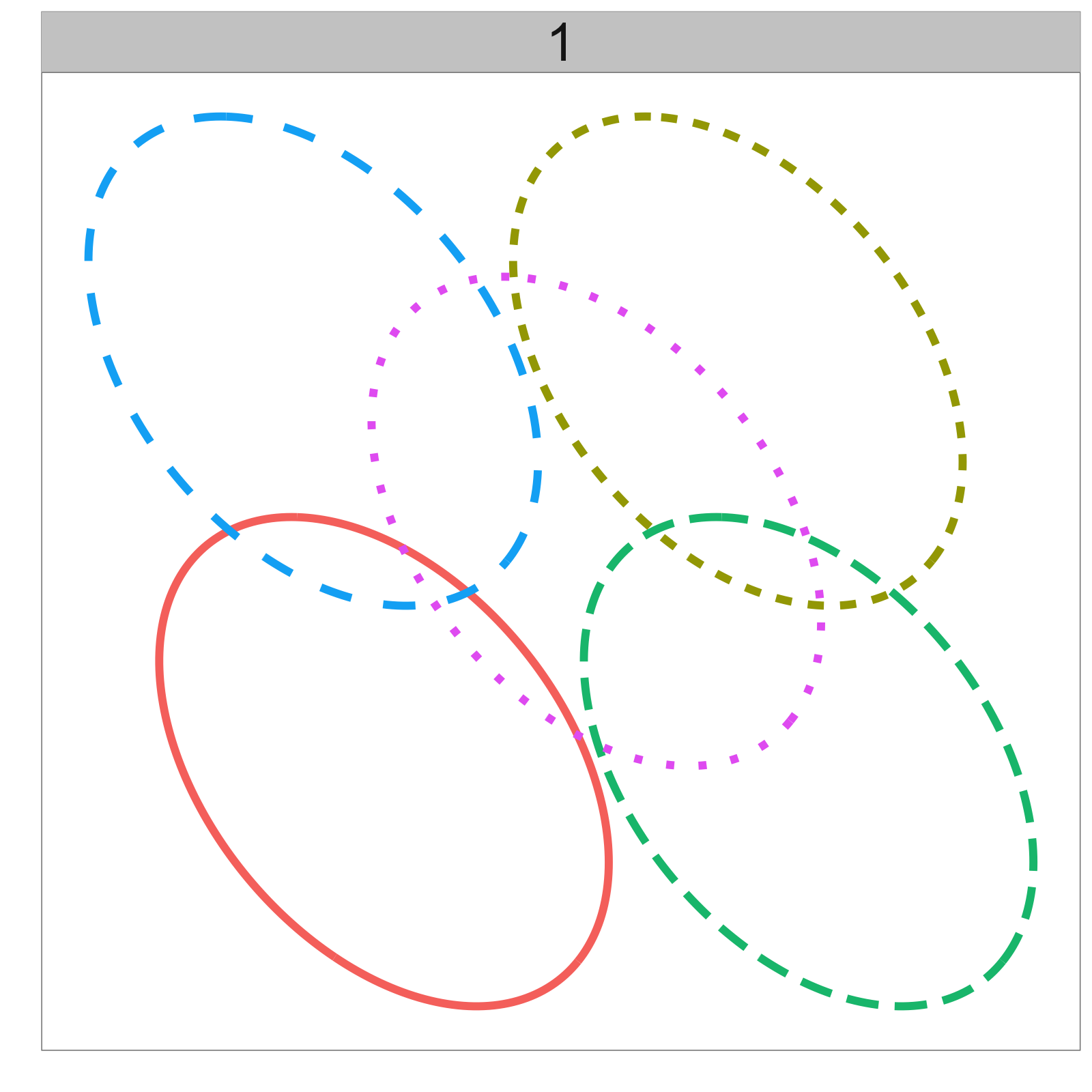} 

\end{knitrout}
\caption{Contours of five multivariate normal populations as a function of the
  pooling parameter $\lambda$.}
\label{fig:hdrda-contours}
\end{figure}

\clearpage

\begin{figure}
\begin{knitrout}
\definecolor{shadecolor}{rgb}{0.969, 0.969, 0.969}\color{fgcolor}
\includegraphics[width=\linewidth]{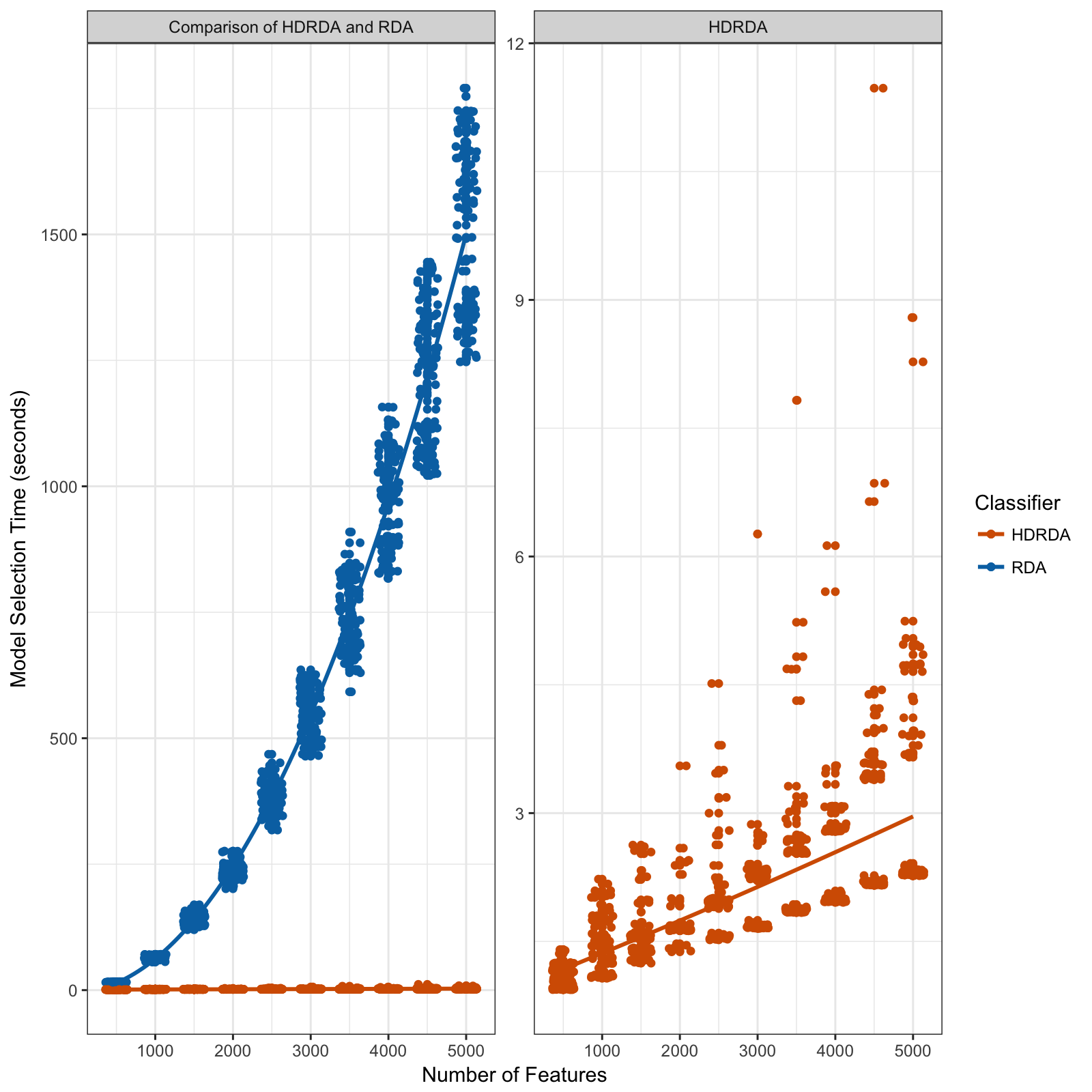} 

\end{knitrout}
\caption{Timing comparisons (in seconds) between HDRDA and RDA classifiers.}
\label{fig:timing-results}
\end{figure}

\clearpage

\begin{figure}
\begin{knitrout}
\definecolor{shadecolor}{rgb}{0.969, 0.969, 0.969}\color{fgcolor}
\includegraphics[width=\linewidth]{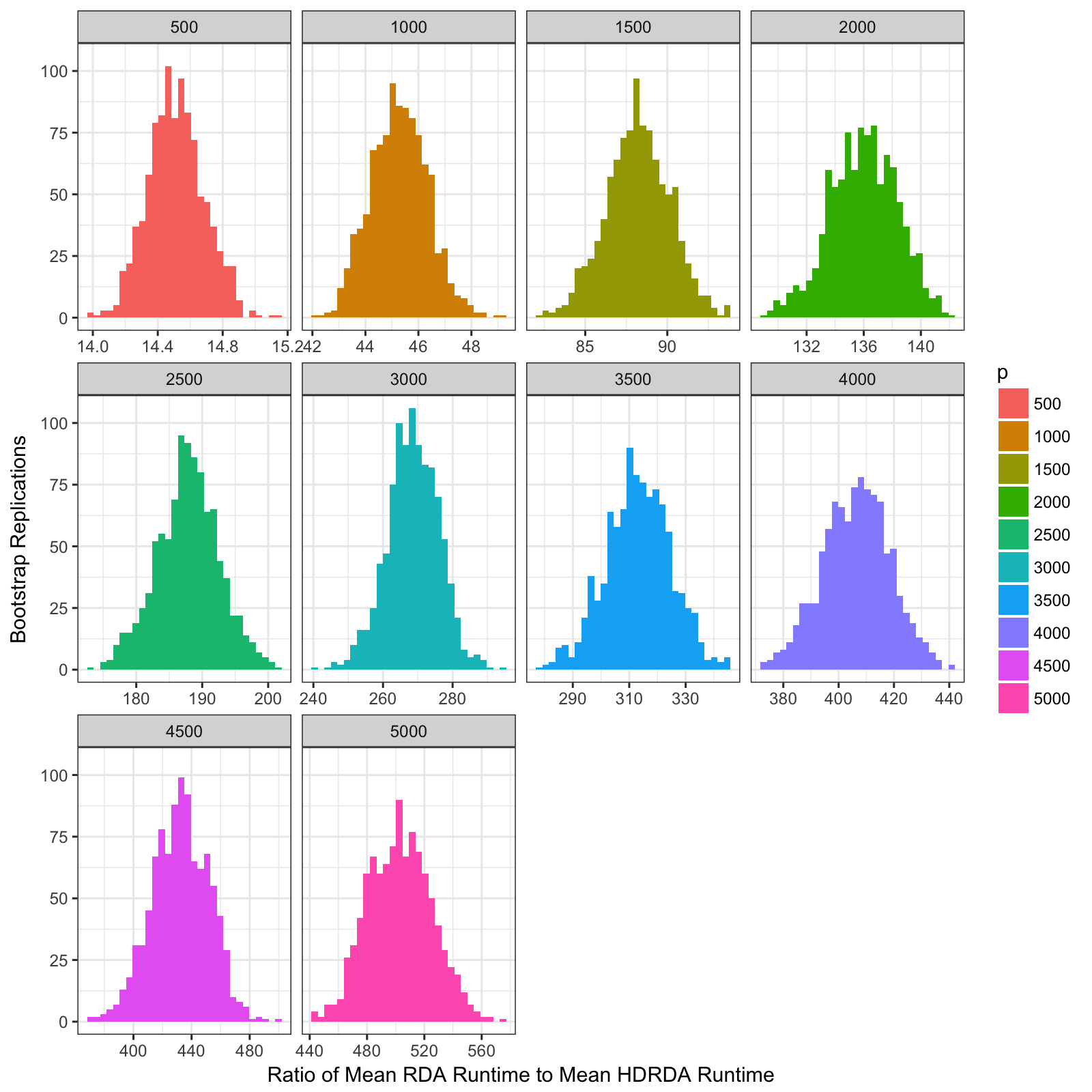} 

\end{knitrout}
\caption{Distribution of ratios of mean \emph{RDA} runtime to mean \emph{HDRDA}
  runtime across 1000 bootstrap replications.}
\label{fig:timing-comparison-bootstrap}
\end{figure}

\clearpage

\begin{figure}
\begin{knitrout}
\definecolor{shadecolor}{rgb}{0.969, 0.969, 0.969}\color{fgcolor}
\includegraphics[width=\linewidth]{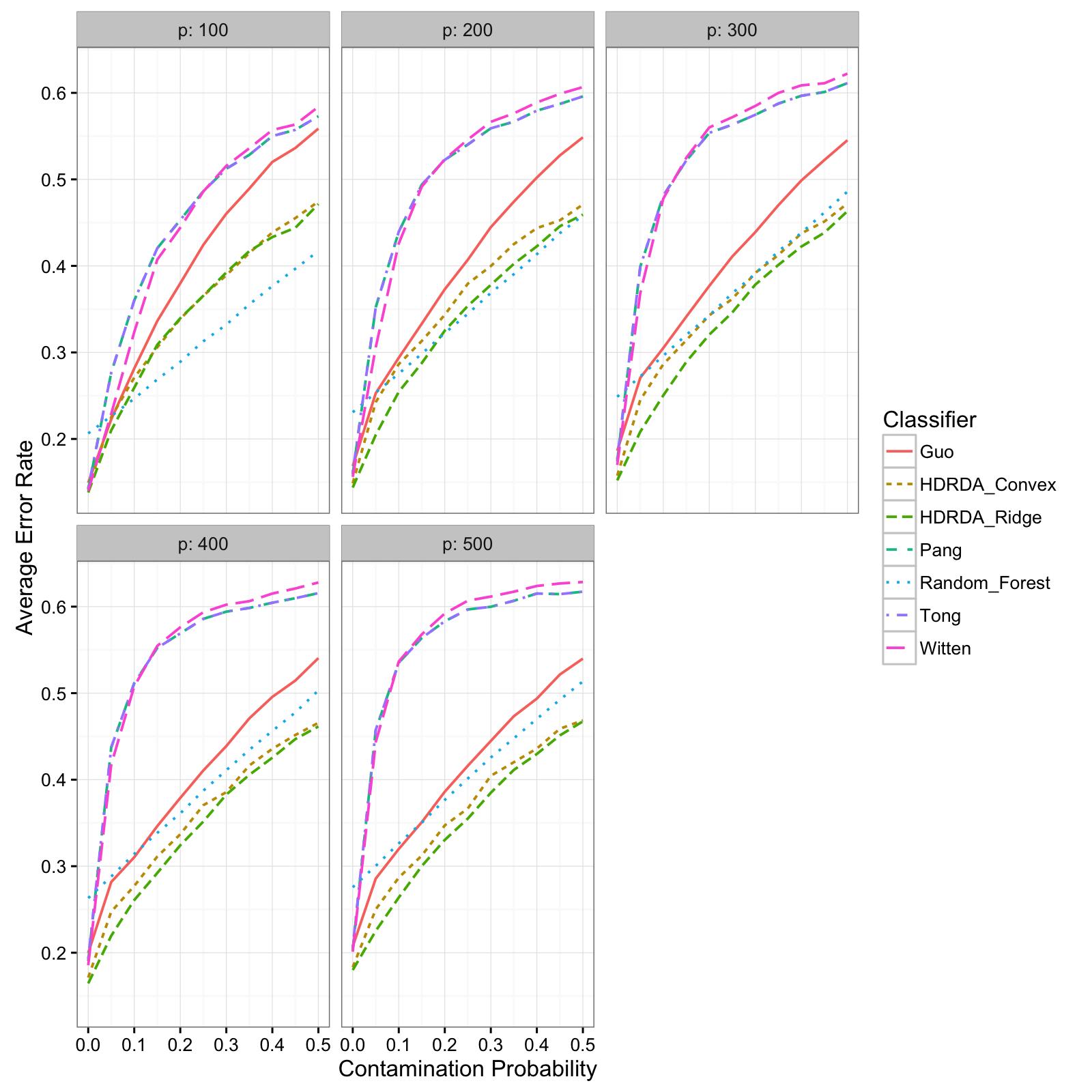} 

\end{knitrout}
\caption{Average classification error rates as a function of the contamination
  probability $\epsilon$. Approximate standard errors were no greater than
  0.022.}
\label{fig:sim-results}
\end{figure}

\clearpage

\begin{figure}
\begin{knitrout}
\definecolor{shadecolor}{rgb}{0.969, 0.969, 0.969}\color{fgcolor}
\includegraphics[width=\linewidth]{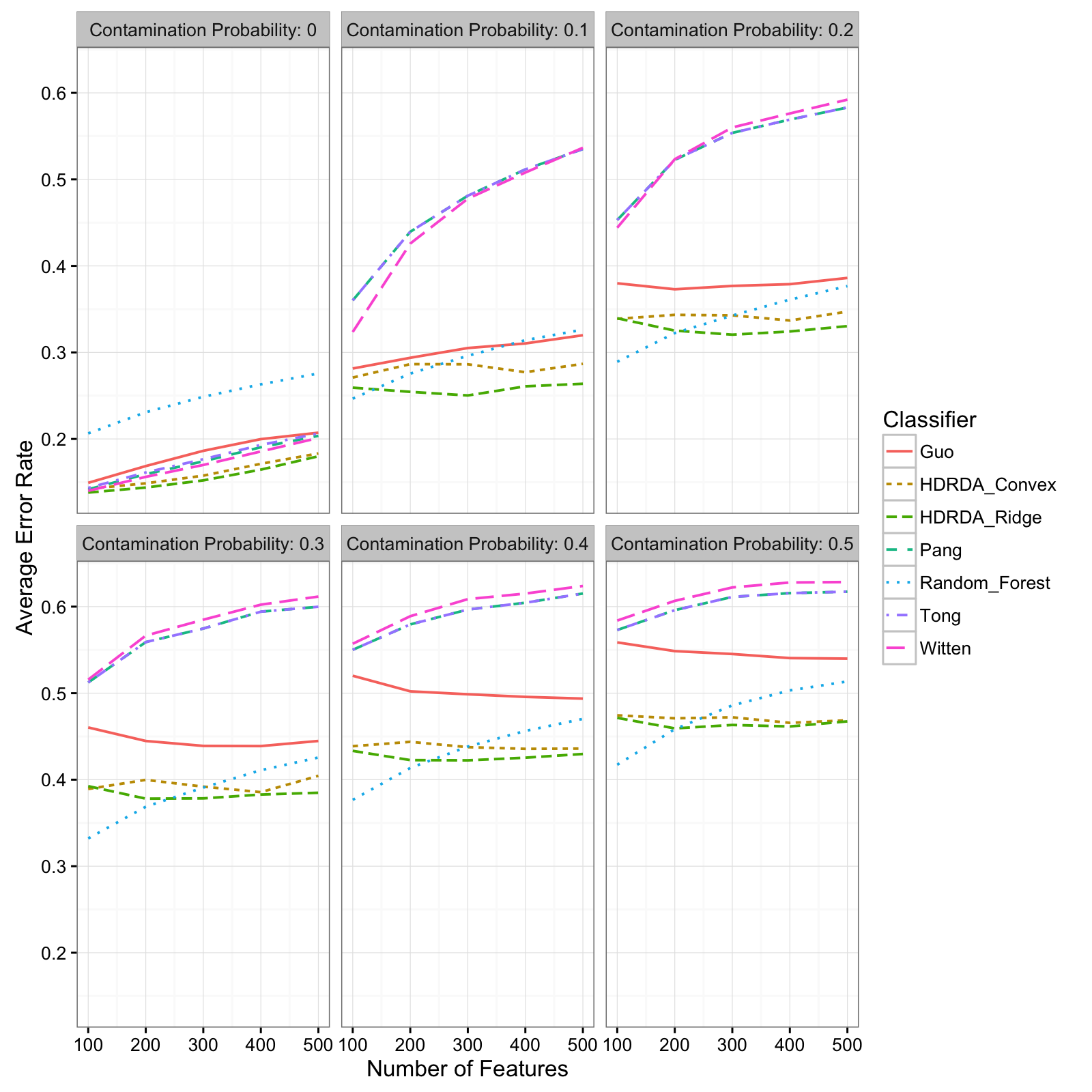} 

\end{knitrout}
\caption{Average classification error rates as a function of the number of
  features $p$. Approximate standard errors were no greater than
  0.022.}
\label{fig:sim-results-by-p}
\end{figure}

\clearpage

\begin{table}[ht]
\centering
\begin{tabular}{lllllll}
  \hline
Classifier & Chiaretti & Chowdary & Nakayama & Shipp & Singh & Tian \\ 
  \hline
Guo & \textbf{0.111} (0.044) & 0.056 (0.051) & 0.208 (0.061) & 0.086 (0.063) & \textbf{0.089} (0.055) & 0.268 (0.082) \\ 
  HDRDA Convex & 0.115 (0.044) & 0.035 (0.026) & \textbf{0.208} (0.066) & 0.073 (0.057) & 0.111 (0.059) & 0.229 (0.049) \\ 
  HDRDA Ridge & 0.118 (0.050) & \textbf{0.033} (0.022) & 0.208 (0.070) & \textbf{0.072} (0.065) & 0.099 (0.046) & 0.225 (0.050) \\ 
  Pang & 0.663 (0.062) & 0.197 (0.091) & 0.227 (0.062) & 0.192 (0.091) & 0.221 (0.095) & 0.267 (0.054) \\ 
  Random Forest & 0.124 (0.053) & 0.045 (0.028) & 0.232 (0.063) & 0.135 (0.078) & 0.093 (0.045) & \textbf{0.206} (0.044) \\ 
  Tong & 0.195 (0.068) & 0.197 (0.091) & 0.227 (0.062) & 0.192 (0.091) & 0.221 (0.095) & 0.267 (0.054) \\ 
  Witten & 0.194 (0.068) & 0.197 (0.091) & 0.232 (0.068) & 0.193 (0.092) & 0.221 (0.095) & 0.264 (0.053) \\ 
   \hline
\end{tabular}
\caption{The average of the test error rates obtained on gene-expression data sets over 100 random training-test partitions. Standard deviations of the test error rates are given in the parentheses. The classifier with the minimum average error rate for each data set is in bold.} 
\label{tab:microarray-results}
\end{table}

\end{document}